\documentclass{article}
\usepackage{amsmath}
\usepackage{graphicx}
\usepackage{enumerate}
\usepackage{enumitem}
\usepackage{url}

\usepackage[margin=1in]{geometry}

\usepackage{amssymb, amsthm,mathtools}

\usepackage{authblk}
\newtheorem{theorem}{Theorem}[section]
\newtheorem{proposition}[theorem]{Proposition}
\newtheorem{assumption}{Assumption}

\newtheorem{corollary}[theorem]{Corollary}
\newtheorem{lemma}[theorem]{Lemma}
\newtheorem{definition}[theorem]{Definition}

\newtheorem{conjecture*}{Conjecture}
\usepackage{algorithmic}
\usepackage{algorithm}

\usepackage[colorlinks]{hyperref}

\usepackage{url}
\usepackage{subcaption}
\usepackage{makecell}
\usepackage{setspace}
\usepackage{booktabs}
\usepackage{wrapfig}

\newcommand{\lfm}{\texttt{LFM}}

\newcommand{\calP}{\mathcal{P}}

\newcommand{\calN}{\mathcal{N}}

\newcommand{\calD}{\mathcal{D}}

\newcommand{\E}{\mathbb{E}}
\newcommand{\R}{\mathbb{R}}
\newcommand{\W}[0]{\mathcal{W}_2}

\begin{document}

\title{Local Flow Matching Generative Models}

\author[1]{Chen Xu}
\author[2]{Xiuyuan Cheng}
\author[1]{Yao Xie}
\affil[1]{{\small H. Milton Stewart School of Industrial and Systems Engineering, Georgia Institute of Technology.}}
\affil[2]{{\small Department of Mathematics, Duke University}}

\date{\vspace{-20pt}}

\maketitle

\begin{abstract}
Flow Matching (FM) is a simulation-free method for learning a continuous, invertible flow that interpolates between two distributions, and in particular generates data from noise. Inspired by the variational nature of the diffusion process as a gradient flow, we introduce a stepwise FM model, Local Flow Matching (LFM), which sequentially learns a sequence of FM submodels, each matching a diffusion process up to the time-step size in the data-to-noise direction. In each step, the two distributions to be interpolated by the sub-flow model are closer than those in the full-flow matching model, which interpolates data to noise distributions, enabling smaller models with more efficient training. This variational perspective also allows us to prove a theoretical generation guarantee for the proposed flow model in terms of the $\chi^2$-divergence between the generated and true data distributions, leveraging the contraction property of the diffusion process. In practice, the stepwise structure of LFM is naturally amenable to model distillation, and various distillation techniques can be applied to accelerate generation. We empirically demonstrate that LFM achieves competitive generative performance compared to FM on unconditional generation of tabular and image datasets, and on conditional generation of robotic manipulation policies.
\end{abstract}

\section{Introduction}

Generative modeling has revolutionized machine learning by enabling the creation of realistic synthetic data across various domains. 
Earlier generative models, such as Generative Adversarial Networks (GAN)  \cite{goodfellow2014generative}, Variational Autoencoders (VAE)  \cite{Kingma2014}, and Normalizing Flows  \cite{nflow_review} have achieved impressive results but face challenges such as mode collapse and training instability  \cite{saxena2021generative}. 
Recently, diffusion models  \cite{ho2020denoising,song2021scorebased} and closely related flow-based models  \cite{lipman2023flow,albergo2023building,albergo2023stochastic,fan2022variational,xu2023normalizing} have emerged as competitive alternatives to these earlier approaches (e.g.,  Generative Adversarial Networks (GAN)  \cite{goodfellow2014generative},
Variational Autoencoders (VAE)  \cite{Kingma2014} and Normalizing Flows  \cite{nflow_review}), offering notable advantages in stability, diversity, and scalability.

Diffusion models have been successfully applied across various domains,  including audio synthesis  \cite{kong2021diffwave}, text-to-image generation  \cite{rombach2022high}, and imitation learning for robotics  \cite{chi2023diffusionpolicy}, among others. A key advantage of score-based diffusion models is their {\it simulation-free} training, in which the training objective is formulated as a ``score matching'' loss, defined as the squared $L^2$ loss averaged over data samples. This formulation, under a Stochastic Differential Equation (SDE) formulation  \cite{song2021scorebased}, enables the training of a continuous-time score function parametrized by a neural network, even in high-dimensional settings with large datasets. 

Compared to the SDE generation in diffusion models, the Ordinary Differential Equation (ODE) generation of a trained diffusion or flow model is deterministic and typically uses fewer time steps. Leveraging the ODE formulation, Flow Matching (FM) models  \cite{lipman2023flow,albergo2023building,liu2022rectified} introduced a simulation-free training of flow models by regressing vector fields using a squared-$L^2$ ``matching'' loss. These models have achieved state-of-the-art performance across various tasks,
including text-to-image generation  \cite{esser2024scaling}, 
humanoid robot control  \cite{rouxel2024flow}, 
audio generation  \cite{guan2024lafma}, 
and discrete data applications, such as code generation \cite{gat2024discrete}. 
The simulation-free training of flow models significantly reduces the computational challenges faced by earlier Continuous Normalizing Flow (CNF) models that relied on likelihood-based training \cite{grathwohl2018ffjord}. Beyond the efficient generation enabled by ODE flow models, recent advances in {\it model distillation}  \cite{salimans2022progressive,liu2023flow,song2023consistency} 
-- a process where a complex or computationally intensive ``teacher model'' is used to train a simpler or smaller ``student model'' while preserving much of the teacher model's performance-- 
have further accelerated the generation of large generative models (both diffusion and flow types), enabling high-quality results with an extremely small number of steps in some cases.

Recently, a flow-based generative network, called JKO-iFlow, was proposed  \cite{xu2023normalizing}, inspired by the Jordan-Kinderlehrer-Otto (JKO) scheme, which leverages stepwise training to mimic the discrete-time dynamics of the Wasserstein gradient flow. By stacking residual blocks sequentially, JKO-iFlow enables efficient block-wise training, reducing memory usage and addressing the challenges of end-to-end training. The stepwise approach has also been utilized in the literature to achieve computational advantages, including block-wise training of ResNet within the GAN framework  \cite{johnson2019framework} and flow-based generative models implementing discrete-time gradient descent in Wasserstein space based on the JKO scheme  \cite{
mokrov2021large,xu2023normalizing,vidal2023taming}. 
However, JKO-iFlow is not simulation-free, which raises the key research question motivating this work: 

\

Can we develop {\it simulation-free} iterative flow models that retain the benefits of block-wise iterative training, are scalable to high-dimensional data, and are amenable to theoretical analysis with established performance guarantees?

\

In this paper, we propose a simulation-free flow-based generative model called ``Local Flow Matching'' (\lfm{}), which can be viewed as an iterative block-wise version of FM. We refer to previous FM approaches as {\it global} FMs. The proposed \lfm{} trains a sequence of small (sub-)flow models that, when concatenated, transport invertibly between data and noise distributions.
In each sub-flow, any FM model can be plugged in. Unlike global FM, which directly interpolates between noise and data distributions that may differ significantly,  \lfm{} decomposes this task into smaller, incremental steps. Each step interpolates between distributions that are closer to each other (hence ``local''), as illustrated in Figure \ref{fig:lfm_illu}.
Since the source and target distributions in each step are not too far apart, \lfm{} enables the use of smaller models, potentially leading to faster convergence during training. This reduces memory requirements and computational costs without compromising model quality. Beyond its efficiency in training from scratch, the \lfm{} framework is compatible with various distillation techniques. Empirically, we found that our model outperforms global FM models after distillation, with practical advantages.

Specifically, in each step of \lfm{}, we train a sub-flow model to interpolate between $(p_{n-1}, p_n^*)$, where $p_n^*$ is obtained by evolving $p_{n-1}$ along the diffusion process for a time duration corresponding to the step size. The forward process (data-to-noise) begins from $p_0$, the data distribution, and ends at some $p_N$, which is close to the normal distribution. The reverse process (noise-to-data) leverages the invertibility of each sub-flow model to generate data from noise by ODE integration. By constructing each step to use the (marginal distribution of) a diffusion process as a ``target,'' we establish a theoretical foundation for the generation guarantee of \lfm{} by connecting to the diffusion theory.

In summary, the contributions of the work are as follows:

\begin{itemize}

\vspace{2pt}

    \item 
{\it Stepwise Flow-Matching framework}: 
We propose a new flow-based generative model, which trains a consecutive series of FM sub-models to generate data from a normal distribution, approximately following the diffusion process in the data-to-noise direction. The decomposition of the global flow into small steps results in distributions closer to each other at each step, enabling smaller sub-models and faster training convergence. The concatenation of all sub-models provides an invertible flow that operates in the reverse direction, enabling efficient generation from noise to data. 

\vspace{2pt}
\item 
    {\it Theoretical guarantee}: We establish a generation guarantee, demonstrating how the distribution generated by \lfm{} approximates the data distribution $P$ under $\chi^2$ divergence, which further implies Kullback–Leibler (KL) and Total Variation (TV) distance guarantees. Specifically, we prove an $O(\varepsilon^{1/2})$-$\chi^2$ guarantee, where $\varepsilon$ is the $L^2$ error of the FM training objective. These guarantees are derived under technical assumptions motivated by our stepwise FM approach applied to the Ornstein–Uhlenbeck process (OU) process.  
    Our theory extends to the scenarios where $P$ has only finite second moments (using a short-time initial diffusion to smooth $P$), such as when $P$ is compactly supported.

    \vspace{2pt}

    \item 
    {\it Flexibility and empirical performance}:
    Our framework enables flexible incorporation of various FM designs within each local sub-flow model. Additionally, the stepwise structure of \lfm{} is naturally amenable to distillation, making it compatible with different distillation techniques. 
    Empirically, \lfm{} demonstrates improved training efficiency and competitive generative performance compared with existing FM methods across various tasks such as likelihood estimation, image generation, and robotic manipulation policy learning.

\end{itemize}

\paragraph*{Notations}
We use the same notation for the distribution and its density (with respect to the Lebesgue measure on $\R^d$) when there is no confusion.
For a distribution $P$, $M_2(P) := \int_{\R^d} \| x \|^2 dP(x)$. 
Let $\calP_2 = \{ P \text{ on $\R^d$}, \, s.t., M_2(P) < \infty \}$.
The Wasserstein-2 distance, denoted by $\W(P,Q)$, gives a metric on $\calP_2$.
For $T: \R^d \to \R^d$, 
$T_\# P$ denotes the {\it pushforward} of $P$, i.e.,
$T_\# P(A) = P( T^{-1}(A))$ for a measurable set $A$.
We also write $T_\# p$ for the pushforwarded density.

\subsection{Related works}

\paragraph{Continuous normalizing flow (CNF)}
CNF uses a neural ODE model  \cite{chen2018neural} {optimized} by maximizing the model likelihood, which the ODE parametrization can compute, 
on observed data samples  \cite{grathwohl2018ffjord}. 
To facilitate training and inference, subsequent works have proposed advanced techniques such as trajectory regularization  \cite{finlay2020train,onken2021ot} and block-wise training  \cite{fan2022variational,xu2023normalizing}. These techniques help stabilize the training process and improve the model's performance. 
Despite successful applications in time-series analyses  \cite{de2020normalizing}, uncertainty estimation  \cite{berry2023normalizing}, optimal transport  \cite{xu2025computing}, and astrophysics  \cite{langendorff2023normalizing}, a main drawback of CNF is its computational cost since backpropagating the neural ODE in likelihood-based training is expensive (non-simulation free) and not scalable to high dimensions.

\paragraph{Simulation-free flow models} 
Flow Matching (FM) models  \cite{lipman2023flow,albergo2023building,liu2022rectified} are simulation-free and a leading class of generative models. We review the technical details of FM in Section \ref{subsec:review-FM}.
FM methods are compatible with different choices to interpolate two random end-points drawn from the source and target distributions, e.g., straight lines (called ``Optimal Transport (OT) path'') or motivated by the diffusion process (called ``diffusion path'')  \cite{lipman2023flow}.
Later works also considered pre-computed OT interpolation  \cite{tong2023improving}, and stochastic interpolation paths  \cite{albergo2023stochastic}.
More recently,  \cite{geng2025mean} learns the average rather than the instantaneous velocity to allow one-step high-fidelity generation.
All previous works train a \textit{global} flow model to match between the two distributions, which could require a large model that takes a longer time to train. 
In this work, we propose to train multiple smaller flow models. Our approach is compatible with any existing FM method to train these so-called local flows, making it a flexible and extensible framework.

\paragraph{Accelerated generation and model distillation} 
Model compression and distillation have been intensively developed to accelerate the generation of large generative models. 
\cite{baranchuk2021distilling} proposed learning a compressed student normalizing flow model by minimizing the reconstruction loss from a teacher model.
For diffusion models, progressive distillation was developed in  \cite{salimans2022progressive},
and Consistency Models  \cite{song2023consistency} demonstrated high-quality sample generation by directly mapping noise to data.
For FM models,  \cite{liu2023flow} proposed to distill the ODE trajectory into a single mapping parametrized by a network, which can reduce the number of function evaluations (NFE) to one. 
The approach was later effectively applied to large text-to-image generation  \cite{liu2024instaflow}.
More recent techniques to distill FM models include dynamic programming to optimize stepsize given a budget of NFE   \cite{nguyen2024bellman}.
In our work, each local flow model can be distilled into a single-step mapping following  \cite{liu2023flow}, and the model can be further compressed if needed. Our framework is compatible with different distillation techniques. 

\paragraph{Theoretical guarantees of generative models}

Guarantees of diffusion models, where the generation process utilizes an SDE (random) \cite{lee2023convergence,chen2022sampling,chen2023improved, benton2024nearly} or ODE (deterministic) sampler \cite{chen2023restoration,chen2024probability,li2024towards,li2024sharp,huang2025convergence}, 
have been recently intensively developed. 
The proposed LFM model differs fundamentally from diffusion-based generative models. 
In diffusion models, the step count $N$ refers to the number of discrete time points used  in the reverse process (either SDE or ODE sampling).
Existing theories show that, for commonly used discretizations, achieving a target accuracy $\varepsilon$ in KL or TV typically requires a number of steps that grows polynomially with $\varepsilon^{-1}$, with the precise rate depending on the choice of sampler, discretization order, and error metric.
In contrast, in our analysis of LFM,  $N$  represents the number of  {\it blocks}, where each block corresponds to evolving the density along the Wasserstein gradient flow of an OU process for a constant $O(1)$ time interval and training a flow matching model between the adjacent densities. 
After training, each block defines a neural ODE, which could be integrated numerically on a finer time grid.
These numerical integration steps are closer in spirit to diffusion sampling steps, but are distinct from the theoretical notion $N$ of LFM; hence, we do not compare them directly. 
Our convergence analysis leverages 
the exponential contraction induced by the OU process together with control of the flow matching error at each block.
Under this structure, we obtain a logarithmic dependence  $N = O(\log 1/\varepsilon)$
on the target error  $\varepsilon$  in terms of the number of blocks.

There are comparatively fewer theoretical results for flow-based models such as CNFs or FM.
For CNF trained by maximizing likelihood,  \cite{marzouk2024distribution} proved a non-parametric statistical convergence rate rather than algorithmic convergence. 
For FM models, Wasserstein-distance guarantees were obtained in \cite{benton2024error}  and  \cite{gao2024convergence}, the latter including sample complexity analysis.
However, Wasserstein bounds are weaker than KL or TV guarantees, which are more directly relevant for likelihood-based evaluation and information-theoretic interpretation.
Our work provides a $\chi^2$ bound, which yields stronger control and implies both KL and TV convergence.
 Recently, \cite{silveri2024theoretical} proved KL guarantees, but only for an SDE-based variant of the FM introduced in  \cite{albergo2023stochastic}.
 
 The theory of LFM is closest to that of JKO-iFlow \cite{xu2023normalizing}, 
a stepwise-trained CNF motivated by the JKO scheme, which also achieves KL convergence 
with $N = O(\log 1/\varepsilon)$ in a blockwise sense  \cite{cheng2024convergence}.
A key distinction is that \cite{cheng2024convergence} measures the error in each block via the Wasserstein gradient of the KL divergence, whereas our analysis measures learning error via the flow-matching loss, which directly corresponds to the neural network training objective and is therefore closer to practice. Note that to obtain strong $\chi^2$ guarantees, our analysis additionally relies on regularity assumptions on the intermediate densities (Assumption \ref{assump:A1-A2-A3}); we explain the rationale behind these assumptions after introducing them and discuss possible relaxations in Section~\ref{sec:discuss}. Another important difference is that JKO-iFlow adopts likelihood-based training and is not simulation-free, whereas LFM is simulation-free, which makes it a more practical approach. Compared with original (non-stepwise) flow-matching models, LFM also reduces computational and memory costs in practice.

\section{Preliminaries}\label{sec:prelim}

\paragraph{Flow models and neural ODE}
In the context of generative models, the goal is to generate a data distribution $P$, which is typically accessible only through a finite set of training samples.
When $P$ has density, we denote it by $p$. A continuous-time flow model trains a neural ODE  \cite{chen2018neural} to transform a standard distribution $q$, typically $\calN( 0, I)$ (referred to as ``noise''), into the data distribution $p$.
Specifically, a neural ODE model defines a velocity field  $v(x,t; \theta)$ on $\R^d \times [0,T]$  parametrized by a neural network with trainable parameters $\theta$. In a flow model, the solution of the ODE
\begin{equation}\label{eq:ode-cnf}
 \dot x(t) = v( x(t), t; \theta), \quad t \in [0,T],  \quad x(0) \sim \rho_0,
 \end{equation}
 is used, 
 where $\rho_0$ is a distribution, and the law of $x(t)$ is denoted as $\rho_t$. 
 When the ODE is well-posed, it provides a continuous and invertible mapping from the initial value $x(0)$ to the terminal value $x(T)$. The inverse mapping from $x(T)$ to $x(0)$ can be computed by integrating the ODE \eqref{eq:ode-cnf} in reverse time. This setup allows flexibility in setting $\rho_0$ as the noise distribution and $\rho_T$ as the data, or vice versa. 
 Earlier flow models, such as continuous-time CNFs  \cite{grathwohl2018ffjord} train the velocity field $v(x,t;\theta)$ using likelihood-based objectives, where $\rho_0$ is noise and the likelihood of $\rho_T$ is maximized on data samples. However, these approaches are not simulation-free and can be challenging to scale to high-dimensional and large-sized data. 

\paragraph{Likelihood computation}

The flow-based generative model \eqref{eq:ode-cnf} enables the evaluation of the negative log-likelihood (NLL), a metric we use to assess model performance in the experimental sections. Using the instantaneous change-of-variable formula in neural ODEs \cite{chen2018neural}, we know that for a trained flow model $\hat v(x,t;\theta)$ 
on $t\in [0,1]$ that interpolates between $p$ and $q$, the log-likelihood of data $x\sim p$ can be expressed as 
\begin{equation}
\log p(x)
= \log q(x(1)) + \int_0^1 \nabla \cdot \hat v(x(s),s;\theta) ds, 
\label{likelihood}
\end{equation}
where $x(t)$ solves the ODE $\dot x(t) = \hat{v}( x(t), t;\theta )$ with initial condition $x(0) = x$,
and $\nabla \cdot \hat v(\cdot, s;\theta)$ represents the trace of the Jacobian matrix of the network function. 

\paragraph{Ornstein-Uhlenbeck (OU) process}
The OU process in $\R^d$ is governed by the following SDE:
    \(dX_t = - X_t dt + \sqrt{2} dW_t,\)
where the equilibrium density is $q \propto e^{-V}$ with $V(x) = \|x\|^2/2$; $W_t$ is a standard Wiener process (Brownian motion). In other words, $q$ corresponds to the standard normal density $\calN(0,I)$.
Suppose $X_0 \sim \rho_0$, where $\rho_0$ is some initial density at time zero (the initial distribution may not have a density). Let $\rho_t$ denote the marginal density of $X_t$ for $ t >0$.
The time evolution of $\rho_t$ is described by the Fokker–Planck Equation (FPE):  
$\partial_{t}\rho_t = \nabla\cdot(\rho_t \nabla V + \nabla \rho_t)$, $V(x) = \|x\|^2/2$,
which determines $\rho_t$ given the initial value $\rho_0$.
We introduce the operator $({\rm OU})_0^t $ and write 
\begin{equation}\label{eq:def-operator-OU0t}
    \rho_t = ({\rm OU})_0^t \rho_0.
\end{equation}
Equivalently, $\rho_t$ is the probability density of the random vector
$Z_t: = e^{-t} X_0 + \sigma_t Z$, where $\sigma_t^2 := 1-e^{-2 t}$,
 $Z \sim \calN(0,I_d)$ and is independent of $X_0$.

 \paragraph{Jordan-Kinderlehrer-Otto (JKO) scheme}
The OU process solves the continuous-time Wasserstein gradient flow that minimizes the KL-divergence \cite{ambrosio2005gradient},
and thus is closely related to a discrete-time gradient descent dynamic that also minimizes the  KL-divergence
\begin{align}\label{eq:JKO-classic}
    \rho_{n+1} = \text{arg} \min_{\rho\in \calP_2 }  
     {\rm KL} (\rho || q) + \frac{1}{2 \gamma} \W^2(\rho_n, \rho),
\end{align}
which is known as the JKO scheme  \cite{jordan1998variational}.
The scheme \eqref{eq:JKO-classic} computes a sequence of distributions $\rho_n$, $n=0,1,...$ by 
starting from $\rho_0\in \calP_2$ the $\W$ space (probability distribution in $\R^d$ with finite second moment equipped with $\W$ distance), 
where $\gamma > 0$ is the step size. 
The JKO-iFlow model  \cite{xu2023normalizing} implements \eqref{eq:JKO-classic}  in a step-wise flow network, 
solving for the minimization of the KL-divergence in the training objective in each step.
Our \lfm{} model can be viewed as a variant of the JKO scheme that enjoys simulation-free training as FM methods.
Though the step-wise training objective differs,
we expect each step in \lfm{} also pushes the sequence of data distributions closer to the final normal density $q$, similarly to the JKO scheme.
This intuition will be used in proving the convergence of the \lfm{} model in Section \ref{sec:theory}.

\section{Local Flow Matching}\label{sec:method}

The essence of the proposed 
Local Flow Matching (\lfm{}) model is to decompose a single flow from data to noise (and back) into multiple segments and apply FM on each segment sequentially. This section introduces the method, with additional algorithmic details provided in Section \ref{sec:algorithm}.

\begin{figure*}[t]
    \centering
    \hspace{5pt}
    \includegraphics[width=.8\linewidth]{ 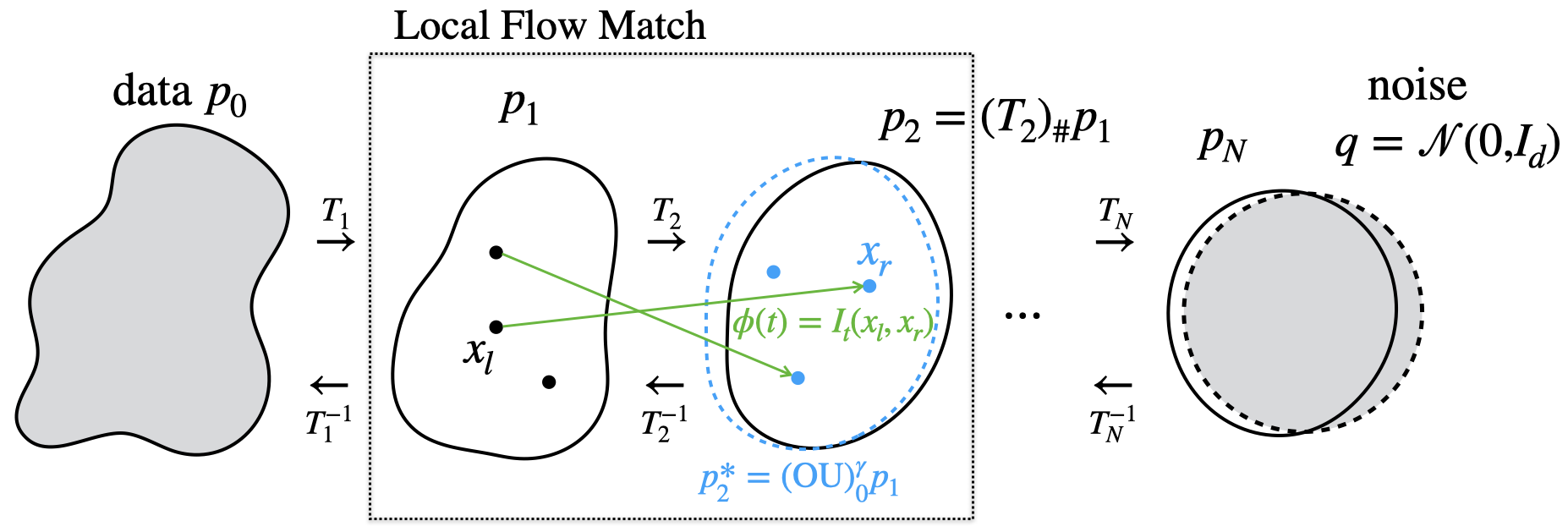}

    \caption{
    Illustration of the proposed \lfm{} model. 
    In the $n$-step step, a local FM model is trained 
    to interpolate from $p_{n-1}$ to $p_n^*$, resulting in the learned (invertible) transport map $T_n$, which pushes forward $p_{n-1}$ to $p_{n}$.
    The concatenation of the $N$ sub-models forms a flow between the data and noise distributions.}
    \label{fig:lfm_illu}
\end{figure*}

\subsection{Review of Flow Matching}\label{subsec:review-FM}

Flow Matching  \cite{lipman2023flow}, 
also proposed as ``stochastic interpolants''  \cite{albergo2023building} and ``rectified flow"  \cite{liu2022rectified}, 
trains the flow $v(x,t;\theta)$ by minimizing a squared-$L^2$ loss and is simulation-free.
Our stepwise approach in this work builds upon FM, and we follow the formulation in \cite{albergo2023building}.
We rescale the time interval to $[0,1]$. 
FM utilizes a pre-specified interpolation function $I_t$ for two endpoints $x_l$ and $x_r$ ($_l$ for ``left'' and $_r$ for ``right'') defined as
\begin{equation}\label{eq:interpolation}
    \phi(t) : =I_t(x_l, x_r), \quad t \in [0,1], 
\end{equation}
where $x_l \sim p$, $x_r \sim q$, and $I_t$ can be analytically designed;
e.g., $I_t $ being a straight line connecting $x_l$ and $x_r$ is called the ``optimal transport'' interpolation,
see more in Section \ref{sec:algorithm}.
Generally, $I_t: \R^d \times \R^d \to \R^d$ can be a $t$-differentiable function satisfying
\begin{equation}\label{eq:It-endpoints-condition}
I_0( x_l, x_r) = \phi(0)= x_l, \quad I_1( x_l, x_r) =\phi(1) = x_r.
\end{equation}
We denote by $\hat v= \hat v(x, t; \theta)$ the learned velocity field parametrized by $\theta$.
 The (population) training objective of FM is given by
\begin{equation}\label{eq:fm_loss}
	L( \hat v ):= 
     \int_0^1 \mathbb{E}_{ x_l, x_r} \| 
 \hat  v( \phi(t),t;\theta)-\frac{d}{dt} \phi(t)\|^2 dt,
\end{equation}
and in practice, the expectation $\E_{x_l, \, x_r}$ is replaced by averaging over sample batches.

It was proved in the literature that the velocity field $v$ that minimizes \eqref{eq:fm_loss} 
induces a flow that transports from $p$ to $q$, and we elaborate more here. Recall how a velocity field $v$  transports particle densities is characterized by the probability flow induced by $v$:
suppose the particle $x(t)$ observes the ODE
\[
\dot x(t) = v(x(t), t), \quad x(0) \sim \rho_0,
\]
and we denote by $\rho_t = \rho(x,t)$ the density of $x(t)$. 
Then $\rho_t$ evolves according to  the continuity equation (CE) written as
\[
\partial_t \rho + \nabla \cdot (\rho v) = 0, \quad \rho( \cdot, 0) = \rho_0.
\]
We are interested in finding $v$ such that from $\rho_0 = p$, at time $t=1$, $\rho_1$ reaches the target density $q$.
Specifically, we introduce the notation of a {\it valid} velocity field:
\begin{definition}
A continuously differentiable velocity field $v(x,t)$ on $\R^d \times [0,1]$ is called  {\it valid} if 
the solution $\rho(x,t)$ of the associated CE $\partial_t \rho + \nabla \cdot (\rho v) = 0$
 starting from $\rho(\cdot,0) = p$ 
 satisfies that $\rho(\cdot,1) = q$.
 \end{definition}

Under some technical conditions (Assumption \ref{assump:technical-rho01-It}), 
the minimizer of \eqref{eq:fm_loss}  provides a valid velocity field. 
This is shown in the following lemma, the same conclusion of  which was proved in \cite[Proposition 1]{albergo2023building} 
where it was assumed that $x_0$ and $x_1$ are independent. 
Our proof of the lemma here mainly shows that the same conclusion holds when allowing dependence between $x_0$ and $x_1$.
The details are provided in Appendix \ref{app:proof-valid-v} for completeness.

\begin{lemma}
\label{lemma:FM_loss}
Given $x_l, x_r$ that are marginally distributed as $p$ and $q$ respectively, suppose the density of $(x_l, x_r)$ and the interpolation function $I_t$ satisfy Assumption \ref{assump:technical-rho01-It}, and then there exists a valid velocity field $v$ such that,
with $\rho_t(x)=\rho(x,t)$ being the solution of the induced CE by $v$, the loss \eqref{eq:fm_loss} can be written as
\begin{equation}\label{eq:loss-FM-2}
L(\hat v) = c + \int_0^1 \int_{\R^d} \| \hat v(x,t) - v(x,t) \|^2 \rho_t(x) dx dt, 
\end{equation}    
where $c$ is a constant independent from $\hat v$. In addition, $\rho_t$ is the marginal density of $x_t = \phi(t)$.
\end{lemma}

It follows directly from \eqref{eq:loss-FM-2} that $v$ is the (unconstrained) minimizer of FM training loss $L$.
We refer to $v$  as the {\it target} velocity field,
and say that the FM model interpolates a pair of distributions $(p,q)$.
While the marginal densities of $x_l, x_r$ are always $p$ and $q$, different ways of introducing dependence between them and different designs of the function $I_t$ will lead to different flows, namely different target $v$, and all of which are valid.

\subsection{Stepwise training of $N$ FM sub-models}\label{sec:formulation}

We propose to train a sequence of $N$ FM sub-models (referred to as sub-flows) over the time interval $[0,T]$,
 each with its own training objective to interpolate between a pair of distributions. Collectively, these $N$ sub-flows achieve the transport from the data distribution $p$ to the noise $q$ distribution (and back via the reverse flow). Our method trains the $N$ sub-flows sequentially, 
where, at each step, the flow aims to match the terminal density evolved by the OU process up to time step size. Since the step size is relatively small and the pair of endpoint distributions to interpolate is close, we refer to our approach as {\it local} FM.

\paragraph{Training}
Suppose we are given a sequence of step sizes $\gamma_n > 0$, and the total time $T  = \sum_n \gamma_n$ is reasonably long for the OU process to approach equilibrium.
Assume that the data distribution $P$ has a regular density $p_0$;  
if not, we can apply a short-time diffusion to $P$, and use the resulting density as $p_0$ (see Section \ref{subsec:theory-reverse-process} for more details).
Starting from $p_0$ (when $n=1$),
we recursively construct target density $p_{n}^*$ by
\(
p_n^* = (\text{OU})_0^{\gamma_n} p_{n-1},
\)
where the operator $(\text{OU})_0^{t}$ is defined in \eqref{eq:def-operator-OU0t}.
In other words, for $x_l \sim p_{n-1}$, we define
\begin{equation}\label{eq:xr-from-xl}
\begin{split}
&x_r : = e^{-\gamma_n } x_l' + \sqrt{ 1-e^{-2 \gamma_n} } g,  \\
& \quad x_l' \sim p_{n-1},  
\quad g \sim  \calN(0,I_d), \quad g \perp x_l',
\end{split}
\end{equation}
where the marginal distribution of $x_r$ is $p_n^*$.
The original FM sets $x_r$ to be independent from $x_l$  \cite{albergo2023building}, which can be implemented by drawing $x_l'$ as an independent copy of $x_l$. 
By the comment beneath Lemma \ref{lemma:FM_loss}, by setting $x_l' = x_l$, which renders $x_r$ and $x_l$ dependent, it also gives a valid (but different) flow in FM.
 In our work, we found that the dependent sampling can give a more regular velocity field in experiments.

We then train the $n$-th sub-flow, denoted by $\hat v_n(x, t; \theta)$, 
using FM to interpolate the pair $(p_{n-1}, p_n^*)$ based on \eqref{eq:interpolation} and \eqref{eq:fm_loss}. During FM training, the time interval $[0, \gamma_n]$ is rescaled to be $[0,1]$.
The velocity field $\hat v_n$ can have its own parametrization $\theta_n$, allowing it to be trained independently of the previous sub-flows.

After training the sub-flow $\hat v_n$, it defines a transport map $ T_{n}$ that maps $x_{n-1} \sim p_{n-1}$ to $x_n$ as follows: 
\begin{equation}\label{eq:def-Tn-2}
x_n := 
   T_{n}(x_{n-1})
    = x( \gamma_n)
     = x_{n-1} + \int_{0}^{\gamma_n} \hat v_n( x(t), t; \theta) dt,
\end{equation}
where $x(t)$ solves the ODE $\dot x(t) = \hat{v}_n( x(t), t;\theta )$ with initial condition $x(0) = x_{n-1}$. 
The mapping $T_n$ is invertible, and its inverse, $T_n^{-1}$, is obtained by integrating the reverse-time ODE. The distribution of $x_n$ is denoted as $p_n$, i.e.,
\(
p_{n} = (T_n)_\# p_{n-1}.
\)
From this $p_n$, the next sub-flow can be trained to interpolate $( p_n, p_{n+1}^*)$. This iterative scheme is illustrated in Figure \ref{fig:lfm_illu}.

If the flow matching in the $n$-th step is successful, we expect the trained $\hat v_n$ to ensure that $p_n \approx p_{n}^*$. Define the time stamps $\{ t_n \}_{n=0}^N$, where $t_0 = 0$, $t_n - t_{n-1} = \gamma_n$, and $t_N = T$.
If $p_n$ exactly equals $p_{n}^*=  (\text{OU})_0^{\gamma_n} p_{n-1}$, then over $N$ steps, we would have $p_N = (\text{OU})_0^{T} p_0$ which approximates the equilibrium $q$ exponentially fast as $T$ increases. 
In practice, the trained flow may exhibit a finite matching error, resulting in a discrepancy between $p_n$ and $p_{n}^*$. However, if the error is small, we would still expect that $p_N \approx q$ provided that $T = \sum_n \gamma_n $ is sufficiently large. 
This will be analyzed theoretically in Section \ref{sec:theory}.

In practice, when we train the $n$-step sub-flow, finite samples of $p_{n-1}$ are obtained by transporting samples from $p_0$ through the previous sub-flows. Once the $n$-th sub-flow is trained, this pushforward from $p_{n-1}$ to $p_n$ can be applied for all training samples, as described in Algorithm \ref{algo:lfm}. 
The proposed \lfm{} model can be trained from scratch and distilled to improve generation efficiency. See Section \ref{sec:algorithm} for further details.

For the discrete time schedule $\gamma_n$, one can set a constant $\gamma_n$ for simplicity. In practice, we use $\gamma_n = c \rho^{n-1}$ for some $c>0$ and $\rho \ge 1$, and one can also design $\gamma_n$ inspired by the time schedule in diffusion models, though our $N$ is very small (less than 10). See more on the choice of $N$ and $\gamma_n $ in Appendix \ref{app:schedule}. 

\paragraph{Generation}
Once the $N$ sub-flows are trained, we generate data from noise by going backward from the $N$-th to the first sub-flows. Specifically, we sample $y_N \sim q$ and compute $y_{n-1} = T_n^{-1} (y_n)$ for $n=N,\cdots,1$,  where $T_n^{-1}$ is obtained by integrating the ODE with velocity field $\hat v_n$ in reverse time. The final output $y_0$ is used as the generated data samples.
The closeness of the distribution of $y_0$ to the data distribution $P$ will be theoretically shown in Section \ref{sec:theory}.

\section{Theoretical guarantee}\label{sec:theory}

In this section, we theoretically analyze how the density generated by the trained \lfm{} model approximates the true data distribution. 
All proofs are provided in Section \ref{app:proof}.

\subsection{Summary of forward and reverse processes}\label{subsec:theory-summary}

Recall that $P$ is the distribution of data in $\R^d$, and $q$ the density of $\calN(0, I_d)$.
The procedures of training and generation in Section \ref{sec:method} can be summarized in the following forward (data-to-noise training) and reverse (noise-to-data generation)  processes, respectively:
 \begin{equation}\label{eq:fwd-bwd-process} 
\begin{split}
\text{(forward)} \quad 
& 
p = p_0  
\xrightarrow{T_1}{p_1}  
\xrightarrow{T_2}{} 
\cdots 
\xrightarrow{T_{N}}{p_N} 
\approx q,  \\ 
\text{(reverse)} \quad 
&  p \approx
q_0  \xleftarrow{T_1^{-1}}{ q_1}
\xleftarrow{T_2^{-1}}{ }
\cdots 
\xleftarrow{T_{N}^{-1}}{ q_N}
= q,
\end{split}
\end{equation}
where $T_n$ is by the learned $n$-th step sub-flow as defined in \eqref{eq:def-Tn-2}.
When $P$ has a regular density $p$ we set it as $p_0$; otherwise, $p_0$ will be a smoothed version of $P$; see more below.
The reverse process gives the final output samples, which have density $q_0$.
Our analysis aims to show that $q_0 \approx p$, i.e., the generated density is close to the data density.

To keep the exhibition simple, we consider when $\gamma_n  = \gamma > 0$ for all $n$.  
Using the time stamps $t_n$, the $n$-th step sub-flow is on the time interval  $[t_{n-1},t_{n}]$, which can be shifted to be $[0, \gamma]$. 
Our analysis will be based on comparing the true or target flow (that transports to terminal density $p_n^{*}$)
with the learned flow (that transports to terminal density $p_n$),
and we introduce the notations for the corresponding transport equations.

For fixed $n$, on the shifted time interval $[0, \gamma]$, the target flow and the learned flow are induced by the velocity field $v(x,t)$ and $\hat v(x,t) = \hat v_n(x,t;\theta)$  respectively. We omit the subscript $_n$ in the notation. 
As was shown in Section \ref{subsec:review-FM},
the target $v$ depends on the choice of $I_t$ yet it always transports from $p_{n-1}$ to $p_n^{*}$.
Let $\rho_t(x,t)$ be the law of $x(t)$ that solves the ODE with velocity field $v$, where $x(0) \sim p_{n-1}$, 
and $\hat \rho_t(x,t)$ be the law of $x(t)$ that solves the ODE with the learned $\hat v$.
(Note that $\rho_t$ is not necessarily the density of an OU process $X_t$, though $\rho_\gamma = p_n^* = (\text{OU})_0^{\gamma_n} p_{n-1}$.)
We also denote $\rho(x,t)$ as $\rho_t$ (omitting the variable $x$) 
or $\rho$ (omitting both $x$ and $t$), depending on the context, and similarly for $\hat \rho(x,t)$.
The target and learned flows have the transport equations as 
\begin{equation}\label{eq:TE-rho-hatrho}
\begin{split}
\partial_t \rho + \nabla \cdot( \rho v) & =0 , \quad \rho_0 = p_{n-1}, \quad \rho_\gamma = p_n^*. \\
\partial_t \hat \rho + \nabla \cdot( \hat \rho \hat v) & =0 , \quad \hat \rho_0 = p_{n-1}, \quad \hat \rho_\gamma = p_n.
\end{split}
\end{equation}

Before we go further into the analysis, we comment on the challenge in analyzing the FM method 
and provide some intuition for our approach. It is well-known that the diffusion process (OU process) contracts exponentially fast, that is, for some constant $c>0$,
\[ 
\calD ( ({\rm OU})_0^t p,  q )\le e^{-c t} \calD( p , q),
\]
where $\calD$ is some measure of distribution distance or discrepancy, such as KL-divergence or $\chi^2$-divergence \cite{bolley2012convergence}.
As a result, if the learning is perfect, i.e., $v = \hat v$ in each step, then $\rho = \hat \rho$ and $p_n = p_n^*$,
which would imply that $\calD( p_N , q) \le \varepsilon$ in $N \lesssim \log ({1}/{\varepsilon})$ steps.
To handle the imperfect learning, 
one is to bound the endpoint density difference $\rho_\gamma - \hat \rho_\gamma = p_n^* - p_n$, under some proper distance measure, 
given that both flows on the $[0,\gamma]$ time interval start from the same initial density $p_{n-1}$.
Unlike in many analyses of Diffusion Models, the lack of noise in the ODE flow here prevents SDE tools from controlling the KL-divergence.
Specifically, the KL-divergence between $\rho_\gamma$ and $\hat \rho_\gamma$ has the expression \cite[Lemma 2.21]{albergo2023stochastic}
\begin{equation}\label{eq:KL-FM-expression}
\begin{split}
{\rm KL}( \rho_\gamma \|  \hat \rho_\gamma )
= \int_0^\gamma \int_{\R^d} 
 &  ( \nabla \log \rho(x,t )  - \nabla \log  \hat \rho(x,t)  ) \\
 & \cdot ( v(x,t) - \hat v(x,t)) \rho(x,t) dx dt,
\end{split}
\end{equation}
and as commented beneath the lemma in \cite{albergo2023stochastic}, a small error in $v - \hat  v$ does not ensure control of the $ \nabla \log  \rho  - \nabla \log \hat \rho $ term and then is generally insufficient to control the KL-divergence.

While \cite{albergo2023stochastic} then introduced a diffusion term (noise in SDE) to obtain a KL-divergence control, here we will stick to the ODE dynamics and proceed with a Cauchy-Schwarz inequality: 
Because the FM training can provide a control of $\int_0^\gamma \int_{\R^d} \| v - \hat v \|^2   {\rho}  dx dt$ say of $O( \varepsilon^2)$,
then ${\rm KL}( \rho_\gamma \|  \hat \rho_\gamma )$
can be bounded if one can control the Fisher divergence 
${\rm FI }(\rho_t \| \hat \rho_t) = \int_{\R^d} \| \nabla \log \rho(x,t )  - \nabla \log \hat  \rho(x,t) \|^2 \rho(x,t) dx$
at all $t$.
This can not be induced from a small $\hat v - v$, but if one only makes mild assumptions on that both 
$\| \nabla \log \rho \|$ and $\| \nabla \log \hat \rho \|$ are $O(1)$, 
then Cauchy-Schwarz would give that ${\rm KL}( \rho_\gamma \|  \hat \rho_\gamma ) = O( \varepsilon)$,
which implies that ${\rm TV}( \rho_\gamma , \hat \rho_\gamma ) = O(\varepsilon^{1/2} )$ by Pinsker's inequality.
While neither Cauchy-Schwarz or Pinsker's is tight, we note that the  $O(\varepsilon^{1/2} )$ control of TV is at the same order as the recent result for the probability flow ODE (from a trained Diffusion model, and $\varepsilon$ there is the $L^2$ score-matching error) \cite{huang2025convergence}, the latter obtained by directly analyzing the TV between $\rho$ and $\hat \rho$ along the flow and using advanced PDE techniques. 

This motivates us to introduce $O(1)$ assumptions on $\| \nabla \log \rho \|$ and $\| \nabla \log \hat \rho \|$
(see Assumption \ref{assump:A1-A2-A3}(A2), technically a linear growth condition),
and to control the divergence between $\rho_\gamma$ and $\hat \rho_\gamma$ based on the $L^2$ control of $v - \hat v$ provided by FM. 
Recall that our design of the \lfm{} model takes a step-wise structure, 
and we would like to control $\calD(p_N, q)$ after $N$ steps/sub-flows. 
Observe that, in the $n$-th step, 

\begin{itemize}

\item[-] 
The exponential contraction over time $\gamma$ can provide a decrease of $\calD( p_n^*, q)$ from $\calD( p_{n-1}, q)$,

\item[-] 
The Cauchy-Schwarz argument above can provide a smallness of $\calD( p_n^*, p_n)$,

\end{itemize}

If one can apply the triangle inequality, then one can guarantee a decrease of $\calD( p_n, q)$ from $\calD( p_{n-1}, q)$ and then iterate. 
This will be the basic idea of our analysis below,
and this step-wise descent analysis also follows the strategy in  \cite{cheng2024convergence}
to prove convergence of the Wasserstein proximal gradient descent scheme corresponding to the JKO-iFlow model  \cite{xu2023normalizing}. 
However, KL-divergence does not satisfy the triangle inequality, so it cannot be used as the $\calD$ for us.
In this work, we choose $\calD$ to be  the $\chi^2$-divergence,
which 
is stronger than the KL divergence (Lemma \ref{lemma:chi2-bound-KL})
and
has a weighted-$L^2$ representation and thus allows the usage of triangle inequality, see Proposition \ref{prop:chi2-fwd-converge}.

\subsection{Exponential convergence of the forward process in $\chi^2$}\label{subsec:theory-converge-fwd}

We introduce a technical condition on the transport map. 
Denote the Lebesgue measure in $\R^d$ as Leb.
\begin{definition}[Non-degenerate mapping]\label{eq:def-ND}
$T: \R^d \to \R^d$ is non-degenerate if for any set $A \subset \R^d$ s.t. $Leb (A ) =0$,
then $Leb(T^{-1} (A) ) = 0$.
\end{definition}

Our first assumption on the learned flow $\hat v_n$ 
is on the smallness of the $L^2$ loss minimization in FM training:

\begin{assumption}\label{assump:A0}
For all $n$,
the learned $\hat v_n$ ensures that $T_n$ and $T_n^{-1}$ are non-degenerate (Definition \ref{eq:def-ND}), 
and, on the time interval $[t_{n-1}, t_n]$ shifted to be $[0, \gamma]$,
$$
\int_0^\gamma \int_{\R^d} \| v - \hat v\|^2   {\rho}  dx dt \le \varepsilon^2.
$$
Without loss of generality, assume $\varepsilon < 1$.
\end{assumption}

As has been shown in Lemma \ref{lemma:FM_loss},
the training objective of  FM is equivalent to minimizing the squared-$L^2$ loss $\int_0^\gamma \int_{\R^d} \| v - \hat v\|^2  \rho  dx dt$.
 Thus, our notion of $\varepsilon$ can be viewed as the learning error of FM (in each step and uniform for all $n$).

Next, following the discussion at the end of Section \ref{subsec:theory-summary}, we introduce some technical assumptions on the flow densities $\rho_t$ and $\hat \rho_t$. 
In particular, (A2) is to facilitate the control of the divergence between $\rho_\gamma$ and $\hat \rho_\gamma$ 
using that of the $\varepsilon$-learning error of FM, 
and (A1)(A3) are essentially requiring a Gaussian decay of the densities over the space.

\begin{assumption}\label{assump:A1-A2-A3}
There are positive constants $C_1, C_2$, and $L$ such that, for all $n$, on the time interval $[t_{n-1}, t_n]$
shifted to $[0, \gamma]$, 

~~~(A1) 
$\rho_t$, $\hat \rho_t$ for any $t \in [0, \gamma]$  are positive on $\R^d$ and
 $\rho_t(x), \, \hat \rho_t(x) \le C_1 e^{-\|x\|^2/2}$;

~~~(A2)
$\forall t \in [0, \gamma]$, $\rho_t, \, \hat \rho_t$ are $C^1$ on $\R^d$
and $ \| \nabla \log \rho_t (x)\|$,  $\| \nabla \log \hat \rho_t  (x) \| \le L (1+ \| x\|)$, $\forall x \in \R^d$;

~~~(A3)
$\forall t \in [0, \gamma]$,  $ \int_{\R^d} (1+ \| x\|)^2  ( {\rho_t^3}/{ \hat \rho_t^2})(x) dx \le C_2$.
\end{assumption}
At $t=0$, $\rho_0 = \hat \rho_0 = p_{n-1}$, thus Assumption \ref{assump:A1-A2-A3} requires that $f = p_n$ for $n=0, 1, \cdots$ satisfies
\begin{equation}\label{eq:assump-on-pn}
\begin{split}
&f (x) \le  C_1 e^{-\|x\|^2/2}, 
\quad 
\| \nabla \log f (x) \| \le L (1+ \| x\|), \\
&~~~
 \int_{\R^d} (1+ \| x\|)^2  f(x) dx \le C_2,
 \end{split}
\end{equation}
by (A1)(A2)(A3), respectively. 
The first condition requires $p_n$ to have a Gaussian decay envelope; the second requires the score of $p_n$ to grow linearly (which can be induced by Lipschitz regularity); and the third inequality follows from the first.
The condition \eqref{eq:assump-on-pn} poses regularity conditions on $p_0$, which can be satisfied by many data densities in applications, and, in particular, if $P$ has finite support, then these hold after $P$ is smoothed by an initial short-time diffusion (Lemma \ref{lemma:compact-supp-P}).

Technically,  Assumption \ref{assump:A1-A2-A3} poses the Gaussian envelope and regularity requirements on all $p_n$ and also $\rho_t$ and $\hat \rho_t$  for all time.
This can be expected to hold at least when FM is well-trained:
suppose $\rho_t$ satisfies (A1)(A2) for all $t$ 
due to the regularity of $v$ (by the analytic $I_t$ and the regularity of the pair of densities $(p_{n-1}, p_n^*)$), when the true and learned flows match each other, we have $\hat \rho \approx \rho$,
thus we also expect (A1)(A2) to hold for $\hat \rho_t$;
Meanwhile, the ratio $\rho_t / \hat \rho_t$ is close to 1 and if can be assumed to be uniformly bounded, 
then (A3) can be implied by the boundedness of $ \int_{\R^d} (1+ \| x\|)^2   {\rho_t}(x) dx$
which can be implied by the Gaussian envelope (A1) of $\rho_t$.

We are ready to prove the convergence of the forward process measured under the $\chi^2$-divergence. 

\begin{proposition}[Exponential convergence of the forward process]\label{prop:chi2-fwd-converge}
Under Assumptions \ref{assump:A0}-\ref{assump:A1-A2-A3}, 
\begin{equation}\label{eq:prop-chi2-fwd-converge}
\chi^2( p_n \| q) \le e^{-2 \gamma n} \chi^2( p_0 \| q) + \frac{C_4}{1 - e^{-2 \gamma } }  \varepsilon^{1/2}
\end{equation}
for $n=1,2,\cdots$, where the constant $C_4$
defined in \eqref{eq:def-C4-proof} below
is determined by $C_1$, $C_2$, $L$, $\gamma$ and $d$.
\end{proposition}

The proof of the proposition follows the idea introduced at the end of Section \ref{subsec:theory-summary}.
A downside of using $\chi^2$-divergence is that the control is worse than with the KL-divergence, namely,
we only obtain $O(\varepsilon^{1/2})$ for $\chi^2$ 
instead of 
$O(\varepsilon)$ for KL. We think this result still illustrates the framework of our analysis
and give a first convergence rate for our model, 
which has room to be improved.

\subsection{Generation guarantee of the backward process}\label{subsec:theory-reverse-process}

\paragraph{$P$ with regular density}
Because the composed transform $ T_N \circ \cdots \circ T_1$ from $p_0$ to $p_N$ is invertible, 
and the inverse map transforms from $q=q_N$ to $q_0$, the smallness of $\chi^2 (p_N \| q_N )$ implies the smallness of $\chi^2 (p_0 \| q_0 )$ due to a bi-directional version of data processing inequality (DPI), see Lemma \ref{lemma:bi-DPI}.
As a result, the exponential convergence of the forward process in Proposition \ref{prop:chi2-fwd-converge} directly leads to the $\chi^2$-guarantee $q_0 \approx p$.

\begin{theorem}[Generation guarantee of regular data density]\label{thm:chi2-guarantee-smooth-p0}
Suppose $P \in \calP_2$ has density $p$. 
Let $p_0 = p$ and  $p_0$ satisfies \eqref{eq:assump-on-pn}.
Under Assumptions \ref{assump:A0}-\ref{assump:A1-A2-A3},
if %
$ 
N \ge \frac{1 }{ 2 \gamma } \left( \log  \chi^2(p_0 \| q)  + \frac{1}{2}\log (1/\varepsilon)  \right)  
\sim \log(1/\varepsilon)$, 
we have 
$\chi^2 (p \| q_0) \le C \varepsilon^{1/2}$, where $C :=(1+\frac{C_4}{1 - e^{-2 \gamma } }  )$.
\end{theorem}

By  ${\rm KL} ( p \| q) \le  \chi^2 ( p \| q)$ (Lemma \ref{lemma:chi2-bound-KL}),
the  $\chi^2$-guarantee of $O(\varepsilon^{1/2})$ in Theorem \ref{thm:chi2-guarantee-smooth-p0}
 implies 
$ {\rm KL}( p \| q_0) = O(\varepsilon^{1/2})$,
and consequently $ {\rm TV}(p, q_0) = O(\varepsilon^{1/4})$ by Pinsker's inequality.
\paragraph{$P$ up to initial diffusion}
For data distribution $P$ that may not have a density or the density does not satisfy the regularity conditions, 
we introduce a short-time diffusion to obtain a smooth density $\rho_\delta$ from $P$ and use it as $p_0$.
This construction can ensure that $\rho_\delta$ is close to $P$, e.g. in $\W$ distance,
and is commonly used in diffusion models  \cite{song2021scorebased}
and also the theoretical analysis (known as ``early stopping'')  \cite{chen2023improved}.

\begin{corollary}[Generation of $P$ up to initial diffusion]\label{cor:short-diffusion}
Suppose $P \in \calP_2$ and for some $\delta <1$,
$p_0 = \rho_\delta = ({\rm OU})_0^\delta (P)$ 
satisfies \eqref{eq:assump-on-pn} for some $C_1$, $C_2$ and $L$. 
With $p_0 = \rho_\delta$, suppose  Assumptions \ref{assump:A0}-\ref{assump:A1-A2-A3} hold,
then for $N$ and $C$ as in Theorem \ref{thm:chi2-guarantee-smooth-p0}, the generated density $q_0$ of the reverse process ensures 
$\chi^2( \rho_\delta || q_0) \le C \varepsilon^{1/2}$.
Meanwhile, $\W( P, \rho_\delta) \le C_5 \delta^{1/2}$, $C_5 : = (M_2(P) + 2  d )^{1/2}$.
\end{corollary}

In particular, as shown in Lemma \ref{lemma:compact-supp-P}, when $P$ is compactly supported (not necessarily possessing density), then  
there exist $C_1$, $C_2$ and $L$ such that for any $\delta >0$,
$p_0 = \rho_\delta$ satisfies \eqref{eq:assump-on-pn}. 
This ensures that $\W( P, \rho_\delta)$ can be made arbitrarily small. 
In general, the theoretical constants  $C_1$, $C_2$ and $L$ may depend on $\delta$, 
and consequently, the constant $C$ in the $\chi^2$ bound also depends on $\delta$. Finally, as noted in the discussion following Theorem \ref{thm:chi2-guarantee-smooth-p0},
 the $O(\varepsilon^{1/2})$-$\chi^2$ guarantee in Corollary \ref{cor:short-diffusion}, which establishes  $q_0 \approx p_\delta $, also implies 
 $O(\varepsilon^{1/2})$-KL 
and $O(\varepsilon^{1/4})$-TV guarantee.

\section{Proofs}\label{app:proof}

\subsection{Proofs in Section \ref{subsec:theory-converge-fwd}}

\begin{proof}[Proof of Proposition \ref{prop:chi2-fwd-converge}]
The proof is based on the descent of $\chi^2( p_n \| q)$ in each step. 
Note that under (A1), by the argument following \eqref{eq:bound-rho2/q-proof} below, we have $p_0/q \in L^2(q)$ i.e. $\chi^2( p_0 \| q) < \infty$.

We first consider the target and learned flows on $[t_{n-1}, t_n]$ to establish the needed descending arguments for the $n$-th step.
We shift the time interval to be $[0,T]$, $T: = \gamma > 0$,
and  the transport equations of $\rho$ and $\hat \rho$ are as in \eqref{eq:TE-rho-hatrho}.
Define
\begin{equation}\label{eq:def-G-hatG}
G(t) := \chi^2(\rho_t \| q), 
\quad
\hat G(t) := \chi^2(\hat \rho_t \| q).
\end{equation}
By the time endpoint values of $\rho$ and $\hat \rho$ as in \eqref{eq:TE-rho-hatrho}, we have
\[
\begin{split}
&~~~
G(0) = \hat G(0) = \chi^2(p_{n-1} \| q),  \\
&
G(T) = \chi^2(p_{n}^* \| q),
\quad 
\hat G(T) = \chi^2(p_{n} \| q).
\end{split}
\]
We will show that a) $G(T)$ descends, and b) $\hat G(T) \approx G(T)$, 
then, as a result, $\hat G(T)$ also descends.

We first verify the boundedness of $G(t)$ and $\hat G(t)$ at all times.
By definition, we have (we write integral on $\R^d$ omitting variable $x$ and $dx$ for notation brevity)
\[
G(t) = \int \frac{(\rho_t -q)^2}{ q} = \int (\frac{\rho_t }{q} -1 )^2 q  = \int (\frac{\rho_t }{q})^2  q  -1,
\]
when the involved integrals are all finite, 
and similarly with $\hat G(t) $ and $\hat \rho_t$. 
To verify integrability, observe that under (A1), %
\begin{equation}\label{eq:bound-rho2/q-proof}
\begin{split}
& \frac{\rho_t^2}{q}(x), \, \frac{\hat \rho_t^2}{q} (x)
	\le \frac{C_1^2 e^{- \| x\|^2}} { c_d e^{- \|x\|^2/2}}
	= \frac{C_1^2}{c_d} e^{- \| x\|^2/2}, \\
& \quad c_d := (2\pi)^{-d/2}.
\end{split}
\end{equation}
This shows that  $\rho_t/q, \hat \rho_t/q$ and thus  $(\rho_t/q-1), (\hat \rho_t/q-1)$ 
are all in $ L^2(q)$ for all $t$.
In particular, this applies to $\rho_0 = \hat \rho_0 = p_0$
by taking $n=1$ and $t=0$.
Thus, $G(t), \, \hat G(t)$ are always finite.
Additionally, \eqref{eq:bound-rho2/q-proof} gives that  
\begin{equation}\label{eq:boundedness-Gt-hatGt} 
G(t), \hat G(t) \le (C_1/c_d)^2  -1 \le (C_1/c_d)^2, \quad \forall t \in [0, T],
\end{equation}

Our descending argument  is based on the following two key lemmas:
\begin{lemma}[$\chi^2$-contraction of OU process]\label{lemma:Gt-contraction-OU}
$  \chi^2(p_{n}^* \| q) \le e^{-2 \gamma} \chi^2(p_{n-1} \| q)$.
\end{lemma}
This implies that $G(T) \le e^{-2 T} G(0)$.

\begin{lemma}[$\hat G(T) \approx G(T)$]\label{lemma:chi2-close}
$\chi^2(p_{n} \| q)  \le \chi^2(p_{n}^* \| q)  + C_4 \varepsilon^{1/2}$, with $C_4$ defined in \eqref{eq:def-C4-proof}.
\end{lemma}

Lemma \ref{lemma:Gt-contraction-OU} follows standard contraction results  of the diffusion process,
and Lemma \ref{lemma:chi2-close} utilizes the approximation of $\hat \rho \approx \rho$ due to the learning of the velocity field $\hat v \approx v$.
We postpone the proofs of the two lemmas after the proof of the proposition.

With the two lemmas in hand, we can put the $n$ steps together and prove \eqref{eq:prop-chi2-fwd-converge}.
Define 
\[
E_n := \chi^2( p_n \| q), \quad  \beta:= e^{-2 \gamma } < 1, \quad \alpha : = C_4 \varepsilon^{1/2},
\]
and then Lemmas \ref{lemma:Gt-contraction-OU}-\ref{lemma:chi2-close} give that 
\[
E_n \le  \beta E_{n-1} + \alpha. 
\]
By induction, one can verify that
\[
E_n 
\le \beta^n E_0 + \frac{\alpha(1-\beta^n)}{1-\beta}
\le \beta^n E_0 + \frac{\alpha}{1-\beta},
\]
which proves \eqref{eq:prop-chi2-fwd-converge} and finishes the proof of the proposition.
\end{proof}

\begin{proof}[Proof of Lemma \ref{lemma:Gt-contraction-OU}]
The lemma follows the well-understood contraction of $\chi^2$ divergence of the OU process, see, e.g., \cite{bolley2012convergence}.
Specifically, 
on the time interval  $[0, T]$, $T=\gamma$, 
the initial density $p_{n-1}$ renders $\chi^2(p_{n-1} \| q) < \infty$, because $p_{n-1}/q = \rho_0/q \in L^2(q)$ by the argument following \eqref{eq:bound-rho2/q-proof}.
For the OU process, since the equilibrium density $q \propto e^{-V}$ with $V(x) = \|x\|^2/2$,
we know that $q$ is strongly convex and then satisfies the Poincar\'e inequality (PI) with constant $C=1$.
By the argument in Eqn. (3) in \cite{bolley2012convergence}, 
the PI implies the contraction claimed in the lemma.
\end{proof}

\begin{proof}[Proof of Lemma \ref{lemma:chi2-close}]
We prove the lemma by showing the closeness of $\chi^2(p_{n} \| q) = \hat G(T)$ to $\chi^2(p_{n}^* \| q) =  G(T)$.
By definition \eqref{eq:def-G-hatG},
\[
G(t) = \| \frac{\rho_t}{q} -1\|_{L^2(q)}^2, 
\quad
\hat G(t) = \| \frac{\hat \rho_t}{q} -1\|_{L^2(q)}^2, \quad \forall t \in [0,T].
\]
We will use the triangle inequality $| \sqrt{G(t)} - \sqrt{\hat G(t)} | \le \| \frac{\hat \rho_t}{q} -  \frac{ \rho_t}{q} \|_{L^2(q)}$.
Observe that 
\begin{align}
\| \frac{\hat \rho_t}{q} -  \frac{ \rho_t}{q} \|_{L^2(q)}^2
& = \int \frac{ (\hat \rho_t - \rho_t)^2 }{q}
= \int \frac{ (\hat \rho_t - \rho_t)^2 }{\hat \rho_t} \frac{\hat \rho_t}{q} \nonumber \\ 
& \le \frac{C_1}{c_d}  \int   \frac{ (\hat \rho_t - \rho_t)^2 }{\hat \rho_t},
\end{align}
where the  inequality is due to the pointwise bound 
$\frac{\hat \rho_t(x)}{q(x)} \le \frac{C_1}{c_d}$, which follows  by (A1) and the expression of $q$ with $c_d$ as in \eqref{eq:bound-rho2/q-proof}.
We introduce
\begin{equation}\label{eq:def-Ft-proof}
F(t) := \int   \frac{ (\hat \rho_t - \rho_t)^2 }{\hat \rho_t}  
= \int (\frac{\rho_t }{\hat \rho_t}-1)^2 \hat \rho_t
\ge 0,
\end{equation}
and then we have
\begin{equation}\label{eq:bound-GT-hatGT-by-F}
| \sqrt{G(T)} - \sqrt{\hat G(T)} | \le \| \frac{\hat \rho_t}{q} -  \frac{ \rho_t}{q} \|_{L^2(q)}
\le ( C_1/c_d)^{1/2} \sqrt{F(T)}.
\end{equation}
Next, we will bound $F(T)$ to be $O(\varepsilon)$, and this will lead to the desired closeness of $\hat G(T)$ to $G(T)$.

\vspace{5pt}

\noindent
$\bullet$ Bound of $F(T)$:

We will upper bound $F(T)$ by differentiating $F(t)$ over time. By definition \eqref{eq:def-Ft-proof},
$F(0) = 0$, and (omitting $_t$ in $\rho_t$ and $\hat \rho_t$ in the equations)
\begin{align*}
\frac{d}{dt} F(t)
& =  \int  (\frac{\rho}{\hat \rho} -1)^2 \partial_t \hat \rho 
	+ 2  (\frac{\rho}{\hat \rho} -1) ( \partial_t \rho - \frac{\rho}{\hat \rho} \partial_t \hat \rho)  \\
& = 	\int 2 \partial_t \rho  (\frac{\rho}{\hat \rho} -1)  
	- \partial_t \hat \rho ( ( \frac{\rho}{\hat \rho} )^2 - 1 ) \\
& = 	\int - 2 \nabla \cdot (\rho v)   (\frac{\rho}{\hat \rho} -1)  
	+   \nabla \cdot ( \hat \rho \hat v) ( ( \frac{\rho}{\hat \rho} )^2 - 1 )  \\
&	~~~ \quad \text{(by transport equations \eqref{eq:TE-rho-hatrho})}\\
& = 	\int 2  (\rho v)   \cdot \nabla (\frac{\rho}{\hat \rho} )  
	- ( \hat \rho \hat v) \cdot 2\frac{\rho}{\hat \rho}  \nabla  ( \frac{\rho}{\hat \rho} )  \\
& = 	2 \int  \rho (v - \hat v)   \cdot \nabla (\frac{\rho}{\hat \rho} )  \\
& = 2 \int (v- \hat v) \frac{\rho^2}{\hat \rho} \cdot ( \nabla \log \rho -  \nabla \log \hat \rho),
\end{align*}
and the last row is by that 
$\nabla (\frac{\rho}{\hat \rho} ) = \frac{\rho }{\hat \rho} (\nabla \log \rho - \nabla \log \hat \rho) $.
Thus,
\[
\begin{split}
\frac{1}{2}F(T)  
&= \frac{1}{2} \int   F'(t) dt \\
&= \int_0^T \int_{\R^d} (v- \hat v) \cdot ( \nabla \log \rho -  \nabla \log \hat \rho) \frac{\rho^2}{\hat \rho}  dx dt.
\end{split}
\]
Then, by Cauchy-Schwarz, we have
\begin{equation}\label{eq:FT-CS-proof}
\begin{split}
\frac{1}{2} F(T)
& \le  
\left(\int_0^T  \int_{\R^d} \| v- \hat v\|^2 {\rho}   dx dt\right)^{1/2} \\
& \left(\int_0^T  \int_{\R^d} \| \nabla \log \rho -  \nabla \log \hat \rho) \|^2 \frac{\rho^3}{\hat \rho^2}  dx dt \right)^{1/2}.
\end{split}
\end{equation}
By the flow-matching error assumption \ref{assump:A0}, we have the first factor in the r.h.s. of \eqref{eq:FT-CS-proof}
 bounded by $\varepsilon$.
Meanwhile, 
the technical conditions (A2)(A3) together imply that  for all $t \in [0,T]$,
\[
\begin{split}
& \int_{\R^d}  \| \nabla \log \rho_t    - \nabla \log \hat \rho_t   \|^2  ( {\rho_t^3}/{ \hat \rho_t^2}) dx  \\
&~~ \le (2L)^2 \int_{\R^d} (1+ \| x\|)^2  ( {\rho_t^3}/{ \hat \rho_t^2})(x) dx 
\le (2L)^2 C_2,
\end{split}
\]
thus the second factor in \eqref{eq:FT-CS-proof} 
is upper bounded by $\sqrt{ (2L)^2 C_2 T}$.
Putting together, we have
\begin{equation}\label{eq:bound-FT-O(eps)-proof}
F(T) \le 2  (2L) \sqrt{C_2T} \varepsilon = C_3 \varepsilon,
\quad C_3: =  4L \sqrt{C_2 \gamma}.
\end{equation}

\vspace{5pt}

\noindent
$\bullet$ Bound of $\hat G(T)- G(T)$:

With the bound \eqref{eq:bound-FT-O(eps)-proof} of $F(T)$, we are ready to go back to \eqref{eq:bound-GT-hatGT-by-F}, which gives
\[
\sqrt{\hat G(T)} \le \sqrt{G(T)} + ( C_1/c_d)^{1/2} \sqrt{F(T)}.
\]
Together with that $G(T)\le (C_1/c_d)^2$ by \eqref{eq:boundedness-Gt-hatGt}, this gives that
\[
\begin{split}
\hat G(T) 
& \le G(T) + 2 (C_1/c_d)^{3/2} \sqrt{C_3 } \varepsilon^{1/2}+  ( C_1/c_d)C_3  \varepsilon \\
& \le G(T) + C_4 \varepsilon^{1/2},
\end{split}
\]
where we used that $\varepsilon < 1$ by Assumption \ref{assump:A0} and
\begin{equation}\label{eq:def-C4-proof}
 C_4 := 2 (C_1/c_d)^{3/2} \sqrt{C_3 }  +  (C_1/c_d) C_3.
 \end{equation}
The fact that $\hat G(T)  \le G(T) + C_4 \varepsilon^{1/2}$ proves the lemma.
\end{proof}

\subsection{Proofs in Section \ref{subsec:theory-reverse-process}}

\begin{lemma}[Bi-direction DPI]\label{lemma:bi-DPI}
Let ${\rm D}_f$ be an $f$-divergence.
    If $T: \R^d \to \R^d$ is invertible 
    and for two densities $p$ and $q$  on $\R^d$,
    $T_\# p$ and $T_\# q$ also have densities, then 
    \[
    D_f(p || q)  = D_f ( T_\# p || T_\# q) . 
    \]
\end{lemma}
\begin{proof}[Proof of Lemma \ref{lemma:bi-DPI}]
Let $X_1 \sim p$, $X_2 \sim q$, and 
$Y_1 = T(X_1)$, $Y_2 = T(X_2)$.
Then $Y_1$ and $Y_2$ also have densities, $Y_1 \sim \tilde p:= T_\# p$ and $Y_2 \sim \tilde q := T_\# q$.
By the classical DPI of $f$-divergence (see, e.g., the introduction of \cite{raginsky2016strong}), 
we have $D_f( \tilde p || \tilde q) \le D_f( p || q)$.
In the other direction, $X_i = T^{-1} (Y_i)$, $i=1,2$, then DPI also implies 
$D_f( p || q) \le D_f( \tilde p || \tilde q)$.
\end{proof}

\begin{proof}[Proof of Theorem \ref{thm:chi2-guarantee-smooth-p0}]
Under the assumptions,  Proposition  \ref{prop:chi2-fwd-converge} applies to give that
\[
\chi^2( p_N \| q) \le e^{-2 \gamma N} \chi^2( p_0 \| q) + \frac{C_4}{1 - e^{-2 \gamma } }  \varepsilon^{1/2}.
\]
Then, whenever $e^{-2 \gamma N} \chi^2( p_0 \| q) \le \varepsilon^{1/2}$ which is ensured by the $N$ in the theorem, we have 
$\chi^2( p_N \| q) \le  (1+\frac{C_4}{1 - e^{-2 \gamma } }  )\varepsilon^{1/2}$.
Let $T_1^N : =  T_N \circ \cdots \circ T_1$, which is invertible, and 
$p_N = (T_1^N)_\# p_0$, $q_N = (T_1^N)_\# q_0$.
We will apply the bi-directional DPI Lemma \ref{lemma:bi-DPI} to show that
\begin{equation}\label{eq:chi2-p0q0=chi2-pNqN}
\begin{split}
\chi^2( p_0 \| q_0) 
&= \chi^2( (T_1^N)_\# p_0 \| (T_1^N)_\# q_0)  \\
&= \chi^2( p_N \| q_N) = \chi^2( p_N \| q),
\end{split}
\end{equation}
which then proves the theorem.

For Lemma \ref{lemma:bi-DPI} to apply to show the first equality in \eqref{eq:chi2-p0q0=chi2-pNqN}, 
it suffices to verify that $p_0$, $q_0$, $p_N$, $q_N$ all have densities. 
$p_0$ has density by the theorem assumption.
From Definition \ref{eq:def-ND}, one can verify that a transform $T$ being non-degenerate guarantees that $P$ has density $\Rightarrow  T_\# P$ has density (see, e.g., Lemma 3.2 of  \cite{cheng2024convergence}).
Because all $T_n$ are non-degenerate under Assumption \ref{assump:A0}, 
from that $p_0$ has density, we know that all $p_n$ has densities, including $p_N$.
In the reverse direction, $q_N = q$, which is the normal density. 
By that $T_n^{-1}$ are all non-degenerate, we similarly have that all $q_n$ have densities, including $q_0$.
\end{proof}
\begin{lemma}\label{lemma:chi2-bound-KL}
For two densities $p$ and $q$ where $\chi^2 ( p \| q) < \infty$,
 ${\rm KL} ( p \| q) \le  \chi^2 ( p \| q)$.
 \end{lemma}
\begin{proof}
The statement is a well-known fact, and we include an elementary proof of its completeness.
Because $\chi^2 ( p \| q) < \infty$, we have $(p/q-1)$ and thus $p/q \in L^2(q)$, i.e. $\int p^2/ q < \infty$.
By the fact that $\log x \le x-1$ for any $x >0$, 
$\log \frac{p}{q} (x) \le \frac{p}{q} (x) -1$, and then
\[
{\rm KL}(p \| q)
= \int p \log \frac{p}{q}
 \le \int p ( \frac{p}{q}-1) 
 = \int \frac{p^2}{q} -1
 = \chi^2(p \| q).
\]
\end{proof}

\begin{proof}[Proof of Corollary \ref{cor:short-diffusion}]
Let $\rho_t = ({\rm OU})_0^t (P)$.
Because $M_2(P) <\infty$, one can show that $\W( \rho_t, P)^2 = O(t)$ as $t \to 0$.
Specifically, by Lemma C.1 in  \cite{cheng2024convergence}, 
$\W( \rho_t, P)^2 \le t^2 M_2(P) + 2t  d $.
Thus,  for $t < 1$,
$\W( \rho_t, P)^2 \le  (M_2(P) + 2  d ) t $.
This proves $\W( P, \rho_\delta) \le C_5 \delta^{1/2}$.

Meanwhile, $p_0 = \rho_\delta$ satisfies the needed condition in Theorem \ref{thm:chi2-guarantee-smooth-p0}, 
the claimed bound of $\chi^2( \rho_\delta || q_0)$ directly follows from Theorem \ref{thm:chi2-guarantee-smooth-p0}. 
\end{proof}

\begin{lemma}\label{lemma:compact-supp-P}
Suppose $P$  on $\R^d$ is compactly supported, then, $\forall t > 0$, $\rho_t = ({\rm OU})_0^t (P)$ satisfies \eqref{eq:assump-on-pn} for some $C_1$, $C_2$ and $L$ (which may depend on $t$). 
\end{lemma}
\begin{proof}[Proof of Lemma \ref{lemma:compact-supp-P}]
Suppose $P$ is supported on $B_R : = \{ x \in \R^d, \|x\| \le R\}$ for some $R>0$. 

By the property of the OU process, $\rho_t$ is the probability density of the random vector
\begin{equation}\label{eq:Zt-law-OU-proof}
Z_t: = e^{-t} X_0 + \sigma_t Z, \quad Z \sim \calN(0,I_d), \quad X_0 \sim P, \quad Z \perp X_0.
\end{equation}
where $\sigma_t^2 = 1-e^{-2 t}$. 
We will verify the first two inequalities in \eqref{eq:assump-on-pn}, and the third one is implied by the first one. 
(The third one also has a direct proof: 
 $ M_2(\rho_\delta) = \E \| Z_\delta \|^2 = e^{-2\delta} M_2(P) + \sigma_\delta^2 d < \infty$,
then the third condition in \eqref{eq:assump-on-pn} holds 
with $C_2 = 1 + e^{-2\delta} M_2(P) + \sigma_\delta^2 d $.)

The law of $Z_t$ in \eqref{eq:Zt-law-OU-proof} gives that 
\begin{equation}\label{eq:expression-rhot-OU}
\rho_t(x) =  \int_{\R^d}  \frac{1}{(2\pi \sigma_t^2)^{d/2}} e^{- \| x - e^{-t} y \|^2/(2 \sigma_t^2)}  dP(y).
\end{equation}

This allows us to verify the first two inequalities in \eqref{eq:assump-on-pn} with proper $C_1$ and $L$, making use of the fact that $P$ is supported on $B_R$. 
Specifically,  by definition, 
\[
\begin{split}
\nabla \rho_t(x)
= \int_{\R^d}  
& \frac{1}{(2\pi \sigma_t^2)^{d/2}} 
	\left(  - \frac{x - e^{-t} y }{\sigma_t^2} \right)  \\
&	e^{- \| x - e^{-t} y \|^2/(2 \sigma_t^2)}  dP(y),
\end{split}
\]
and then, because $\| x - e^{-t} y \|\le \|x \| + e^{-t} \| y\| \le \|x \| + e^{-t} R  $,
\[
\begin{split}
\| \nabla \rho_t(x) \|
& \le   \frac{\| x  \| +  e^{-t} R}{\sigma_t^2}   
\int_{\R^d}  \frac{1}{(2\pi \sigma_t^2)^{d/2}}  \\
&~~~~~~~~~~~~~~~~~~~~~~~
	e^{- \| x - e^{-t} y \|^2/(2 \sigma_t^2)}  dP(y) \\
&= \frac{\| x  \| +  e^{-t} R}{\sigma_t^2}    \rho_t(x).
\end{split}
\]
This means that
\[
\frac{ \| \nabla \rho_t(x) \| }{\rho_t(x)}
\le \frac{\| x  \| +  e^{-t} R}{\sigma_t^2},
\]
which means that the 2nd condition in \eqref{eq:assump-on-pn} holds with $L = \frac{1}{\sigma_t^2}\max\{1,  e^{-t} R  \} $.

To prove the first inequality, we again use the expression \eqref{eq:expression-rhot-OU}.
By that $\| x \| \le \|x - e^{-t} y \| + \| e^{-t} y\|$, and that $\|y\|\le R$, we have
\begin{align*}
\|x\|^2 
& \le \|x - e^{-t} y \|^2+ 2 \|x - e^{-t} y \| \| e^{-t} y\| + \| e^{-t} y\|^2 \\
& \le  \|x - e^{-t} y \|^2+ 2 ( \|x\| + e^{-t} R) e^{-t} R +  e^{-2t} R^2 \\
& = \|x - e^{-t} y \|^2+ 2 e^{-t} R \|x\| + 3 e^{-2t} R^2, 
\end{align*}
and thus
\begin{align*}
e^{- \|x - e^{-t} y \|^2/(2 \sigma_t^2)}
& \le e^{- \frac{1}{2 \sigma_t^2}  (  \|x\|^2  -  2 e^{-t} R \|x\|  -  3 e^{-2t} R^2) }  \\
& = e^{- \frac{1}{2 \sigma_t^2}    \|x\|^2 + \frac{ e^{-t} R}{ \sigma_t^2}   \|x\|}
	e^{\frac{3 e^{-2t} R^2}{2 \sigma_t^2}  }.
\end{align*}
Inserting in \eqref{eq:expression-rhot-OU}, we have that
\[
\rho_t(x) \le  
 	\alpha_t
	e^{- \frac{1}{2 \sigma_t^2}    \|x\|^2+ \frac{ e^{-t} R}{ \sigma_t^2}   \|x\|},
\quad \alpha_t := \frac{e^{\frac{3 e^{-2t} R^2}{2 \sigma_t^2}  }}{(2\pi \sigma_t^2)^{d/2}}.
\]
For $\rho_t(x) \le C_1 e^{- \| x\|^2/2}$, it suffices to have $C_1$ s.t.
\begin{equation}\label{eq:C1-needed-proof}
\frac{C_1}{ \alpha_t} 
	\ge e^{\frac{1}{2}(1 - \frac{1}{\sigma_t^2}) \| x\|^2 
	+ \frac{ e^{-t} R}{ \sigma_t^2}   \|x\|}.
\end{equation}
Because $t >0$,  $1- \frac{1}{\sigma_t^2} =  \frac{-e^{-2t}}{1-e^{-2t}} < 0$, 
the r.h.s. of \eqref{eq:C1-needed-proof} as a function on $\R^d$ decays faster than $e^{-c \|x\|^2}$ for some $c>0$ as $\| x \| \to \infty$,
and then the function is bounded on $\R^d$.
This means that there is $C_1 >0$ to make \eqref{eq:C1-needed-proof} hold. 
This proves that the first inequality in \eqref{eq:assump-on-pn} holds for $\rho_t$.
\end{proof}

\captionsetup[sub]{labelformat=empty}
\begin{figure*}[t]

    \centering
    \begin{minipage}[t]{0.122\textwidth}
        \includegraphics[width=\linewidth]{ 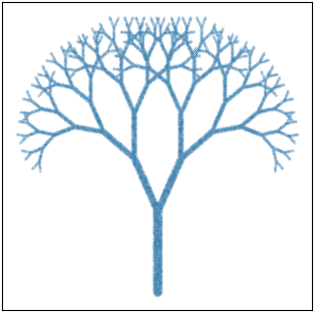}
        \subcaption{\makecell[l]{True data \\ from $p$}}
    \end{minipage}
    \begin{minipage}[t]{0.122\textwidth}
        \includegraphics[width=\linewidth]{ 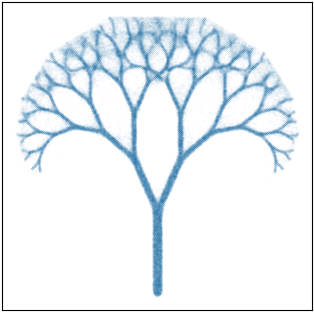}
        \subcaption{\makecell[l]{\lfm{} \\ NLL=\underline{2.24}}}
    \end{minipage}
    \begin{minipage}[t]{0.122\textwidth}
        \includegraphics[width=\linewidth]{ 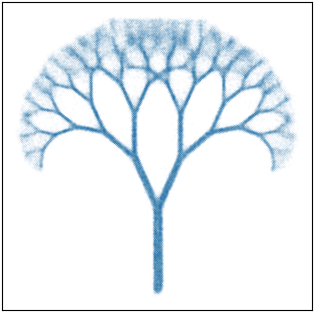}
        \subcaption{\makecell[l]{FM \\ NLL=2.35}}
    \end{minipage}
    \begin{minipage}[t]{0.122\textwidth}
        \includegraphics[width=\linewidth]{ 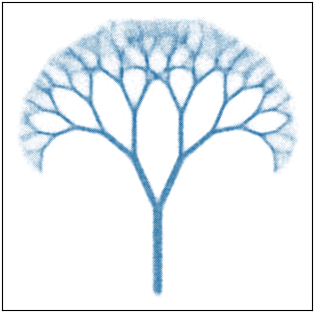}
        \subcaption{\makecell[l]{InterFlow \\ NLL=2.36}}
    \end{minipage}
    \begin{minipage}[t]{0.122\textwidth}
        \includegraphics[width=\linewidth]{ 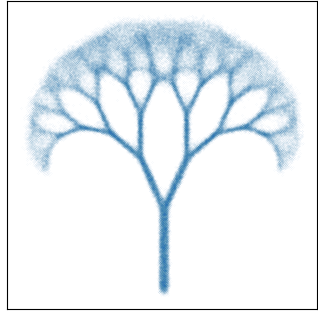}
        \subcaption{\makecell[l]{{ScoreSDE} \\ {NLL=2.51}}}
    \end{minipage}
    \begin{minipage}[t]{0.122\textwidth}
        \includegraphics[width=\linewidth]{ 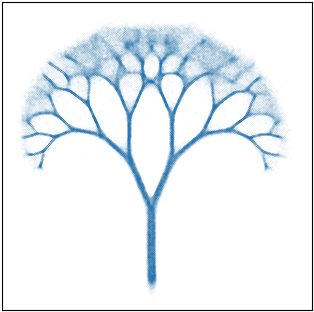}
        \subcaption{\makecell[l]{{Roundtrip} \\ {NLL=2.45}}}
    \end{minipage}
    \begin{minipage}[t]{0.122\textwidth}
        \includegraphics[width=\linewidth]{ 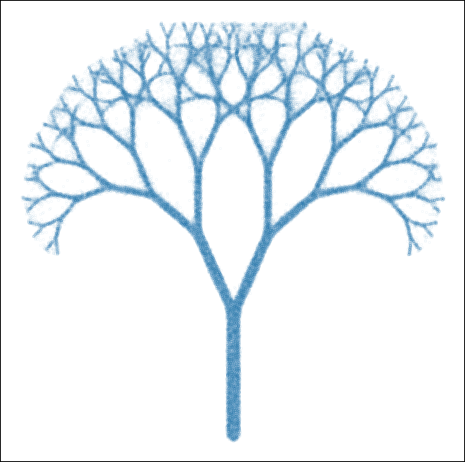}
        \subcaption{\makecell[l]{{JKO-iFlow} \\ {NLL=\textbf{2.11}}}}
    \end{minipage}
    \begin{minipage}[t]{0.122\textwidth}
        \includegraphics[width=\linewidth]{ 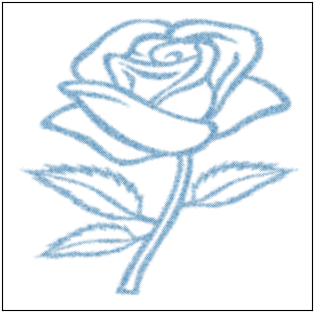}
        \subcaption{\makecell[l]{True data \\ from $p$}}
    \end{minipage}
    \begin{minipage}[t]{0.122\textwidth}
        \includegraphics[width=\linewidth]{ 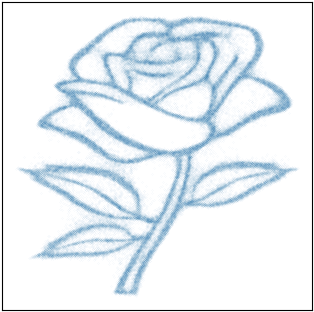}
        \subcaption{\makecell[l]{\lfm{} \\ NLL=\underline{2.60}}}
    \end{minipage}
    \begin{minipage}[t]{0.122\textwidth}
        \includegraphics[width=\linewidth]{ 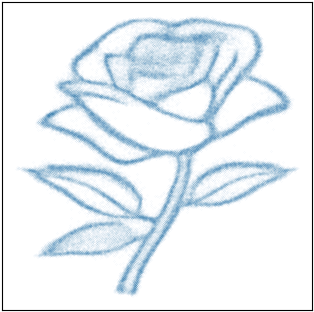}
        \subcaption{\makecell[l]{FM \\ NLL=2.64}}
    \end{minipage}
    \begin{minipage}[t]{0.1255\textwidth}
        \includegraphics[width=\linewidth]{ 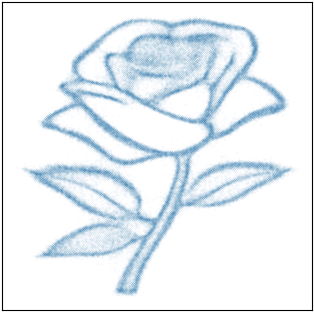}
        \subcaption{\makecell[l]{InterFlow \\ NLL=2.64}}
    \end{minipage}
    \begin{minipage}[t]{0.122\textwidth}
        \includegraphics[width=\linewidth]{ 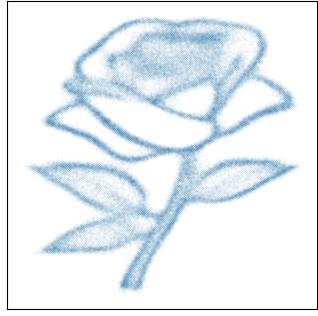}
        \subcaption{\makecell[l]{{ScoreSDE} \\ {NLL=2.72}}}
    \end{minipage}
    \begin{minipage}[t]{0.122\textwidth}
        \includegraphics[width=\linewidth]{ 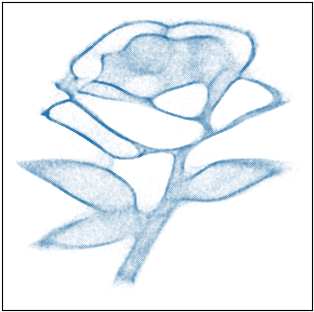}
        \subcaption{\makecell[l]{{Roundtrip} \\ {NLL=2.86}}}
    \end{minipage}
    \begin{minipage}[t]{0.122\textwidth}
        \includegraphics[width=\linewidth]{ 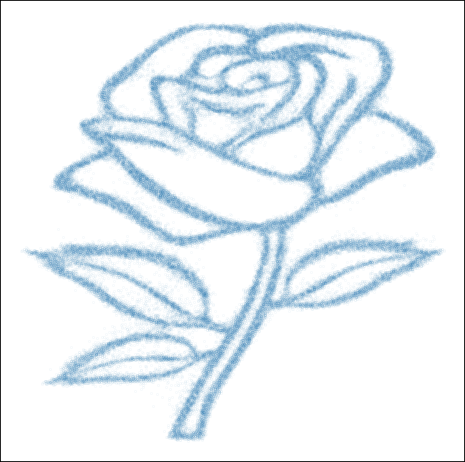}
        \subcaption{\makecell[l]{{JKO-iFlow} \\ {NLL=\textbf{2.53}}}}
    \end{minipage}

    \caption{Generative performance and NLL comparison (lower is better) on 2D data. {The lowest NLL is in \textbf{bold} and the second lowest NLL is \underline{underlined}. 
    For global FM methods, FM  \cite{lipman2023flow} uses the OT interpolant, and InterFlow  \cite{albergo2023building} uses the Trig interpolant, see Section \ref{sec:algorithm}.}
    }
    \label{fig:2d}
\end{figure*}

\begin{table*}[t]
\caption{
Test NLL (lower is better) on tabular data, where $d$ denotes the data dimension. The lowest NLL is indicated in \textbf{bold}, and the second-lowest NLL is \underline{underlined}. Except for ScoreSDE and Roundtrip, where none of the datasets were used, all baseline values are quoted from their original publications, with ``-'' entries indicating that the dataset was not used. For a fair comparison, ScoreSDE and Roundtrip were trained under the same training specification as \lfm{}.
}\label{tab:tabular_results}

\centering
\setstretch{1.25}
\resizebox{\textwidth}{!}{
\fontsize{20pt}{20pt}\selectfont
\begin{tabular}{lccccccccccc}
\toprule
& \lfm{} & InterFlow & JKO-iFlow & AdaCat & {ScoreSDE} & {Roundtrip} & OT-Flow & nMDMA & CPF & BNAF & FFJORD \\
\hline
\makecell{POWER \\ ($d=6$)} & {\underline{-0.715} $\pm$0.004} & {-0.57} & -0.40 & -0.56 & {-0.47} & {4.33} & -0.30 & \textbf{-1.78} & -0.52 & {-0.61} & -0.46 \\
\makecell{GAS \\ ($d=8$)} & {\textbf{-12.445} $\pm$0.057} & \underline{-12.35} & -9.43 & -11.27 & {-11.65} & {1.62} & -9.20 & -8.43 & -10.36 & -12.06 & -8.59 \\
\makecell{MINIBOONE \\ ($d=43$)} & {\underline{10.11} $\pm$0.138} & {10.42} & 10.55 & 14.14 & {10.45} & {18.63} & 10.55 & 18.60 & 10.58 & \textbf{8.95} & 10.43 \\
\makecell{BSDS300 \\ ($d=63$)} & 
{\textbf{-161.16} $\pm$ 0.407}
& -156.22 & \underline{-157.75} & - & {-143.31} & {27.29} & -154.20 & - & -154.99 & -157.36 & -157.40 \\
\bottomrule
\end{tabular}
}
\end{table*}

\section{Experiments}\label{sec:expr}

We apply the proposed \lfm{} to both simulated and real datasets,
including tabular data (Section \ref{sec:tabular}),
image generation (Section \ref{sec:img_gen}), 
and robotic manipulation policy learning (Section \ref{sec:robotics}). 
We demonstrate the improved training efficiency and generative performance of \lfm{} compared to (global) flow models and highlight the advantage of \lfm{} after distillation in the context of image generation. Code is available at 
\url{https://github.com/xycheng/local-flow-matching}.

\subsection{Algorithm and evaluation}\label{sec:algorithm}

We first detail the implementation of \lfm{},
and then introduce evaluation metrics and alternative baselines in our experiments.

\begin{algorithm}[!b]
\caption{Local Flow Matching (\lfm{}) from scratch}
\label{algo:lfm}
\begin{algorithmic}[1]
\REQUIRE 
Data samples $\sim p_0$, timesteps $\{\gamma_n\}_{n=1}^{N-1}$.
\ENSURE{$N$ sub-flows $\{ \hat v_n \}_{n=1}^N$}
\FOR{$n=1,\ldots,N$}
\STATE  
Draw samples $x_l \sim p_{n-1}$ and $x_r \sim p_{n}^* = (\text{OU})_0^{\gamma_n} p_{n-1}$
by \eqref{eq:xr-from-xl} (when $n=N$, let $ p_{n}^* = q$)

\STATE 
Innerloop FM:
optimize $\hat v_n(x,t;  \theta)$ by minimizing \eqref{eq:fm_loss} with mini-batches

\IF{$n\leq N-1$}
\STATE Push-forward the samples $\sim p_{n-1}$ to be samples $\sim p_n$ 
by $T_n$ in \eqref{eq:def-Tn-2}
\ENDIF
\ENDFOR
\end{algorithmic}
\end{algorithm}

\paragraph{Training from scratch}
The training process for \lfm{} is summarized in Algorithm \ref{algo:lfm}. Each sub-flow FM is referred to as a step or a ``block'', and any FM algorithm can be employed in each block. In our experiments, we consider the following choices for $I_t(x_l, x_r)$ in \eqref{eq:interpolation}, following   \cite{lipman2023flow,albergo2023building}:

\begin{itemize}

\item[(i)] Optimal Transport (OT): $I_t(x_l,x_r)=x_l + t(x_r-x_l)$,

\item[(ii)] Trigonometric (Trig): $I_t(x_l,x_r)=\cos(\pi t/2)x_l + \sin(\pi t/2)x_r$.

\end{itemize}

The selection of time stamps $\{\gamma_n\}_{n=1}^{N-1}$ depends on the task, and we use the following scheme: 
\(
    \gamma_n = \rho^{n-1}c, n=1,2,\ldots,
    \)
where $c$ (base time stamp) and $\rho$ (multiplying factor) are user-specified hyper-parameters. In our experiments, we use $N$ up to 10. Each sub-flow $\hat v_n(x,t;\theta)$ has no restrictions on its architecture. We use fully connected networks for vector data and UNets for image data. When sub-flows are independently parametrized, we reduce the model size for each block under the assumption that the target flow to match is simpler than a global flow. This reduces memory load and facilitates the inner-loop training of FM. For computing the pushforward (Line 5 in Algorithm \ref{algo:lfm}) and during generation, the numerical integration of the neural ODE follows the same procedures as in  \cite{chen2018neural}. 

\paragraph{Model distillation}
Inspired by  \cite{liu2023flow}, we further employ distillation on an $N$-block pre-trained \lfm{} model into $N' < N$ step distilled model, where $N = N'k $ for some integer $k$.
Each of the $N'$ sub-models can be distilled independently if parametrized independently. 
The detailed procedure is outlined in Algorithm \ref{algo:lfm_distill}.
If the pre-trained \lfm{} has independently parametrized blocks, we retain the sub-model size for $k=1$ and increase the model size when the distillation combines original blocks ($k >1$). The composition of the $N'$ distilled sub-models generates data from noise in $N'$ steps. 

\paragraph{Evaluation metrics}
We use several performance metrics for the different experiments: 
(i) Negative log-likelihood (NLL): We report NLL in $\log_e$ (known as ``nats'') and following \eqref{likelihood}, 
the evaluation of NLL at test sample $x\sim p$ are using the trained \lfm{} with $N$ sub-flows $\hat{v}_n( \cdot ,t;\theta)$ for $n=1,\ldots,N$. 
Both the integration of $\hat{v}_n$ and $\nabla \cdot \hat{v}_n$ are computed using the numerical integration of neural ODE.  
(ii) For image generation tasks, we also use the commonly adopted Frechet inception distance (FID) \cite{heusel2017gans}. 
(iii) For policy learning, we report an application-specific ``success rate'' as defined in Appendix \ref{sec:robomimic}.

\paragraph{Alternative baselines}
We compare with a series of alternative methods, including various flow- and diffusion-based models.
Within flow models, we mainly compare against 
FM \cite{lipman2023flow},
InterFlow  \cite{albergo2023building},
and $K$-rectified flow  \cite{liu2023flow}.
FM always uses the OT interpolant, and in our experiments, InterFlow always uses the Trig interpolant.
Note that  InterFlow with OT choice of $I_t$ is the same as FM, and 1-rectified flow is also the same as FM. More details of the alternative baselines can be found in Appendix \ref{app:baselines}.

\begin{figure*}[t]
    \centering
    \includegraphics[width=\linewidth]{ 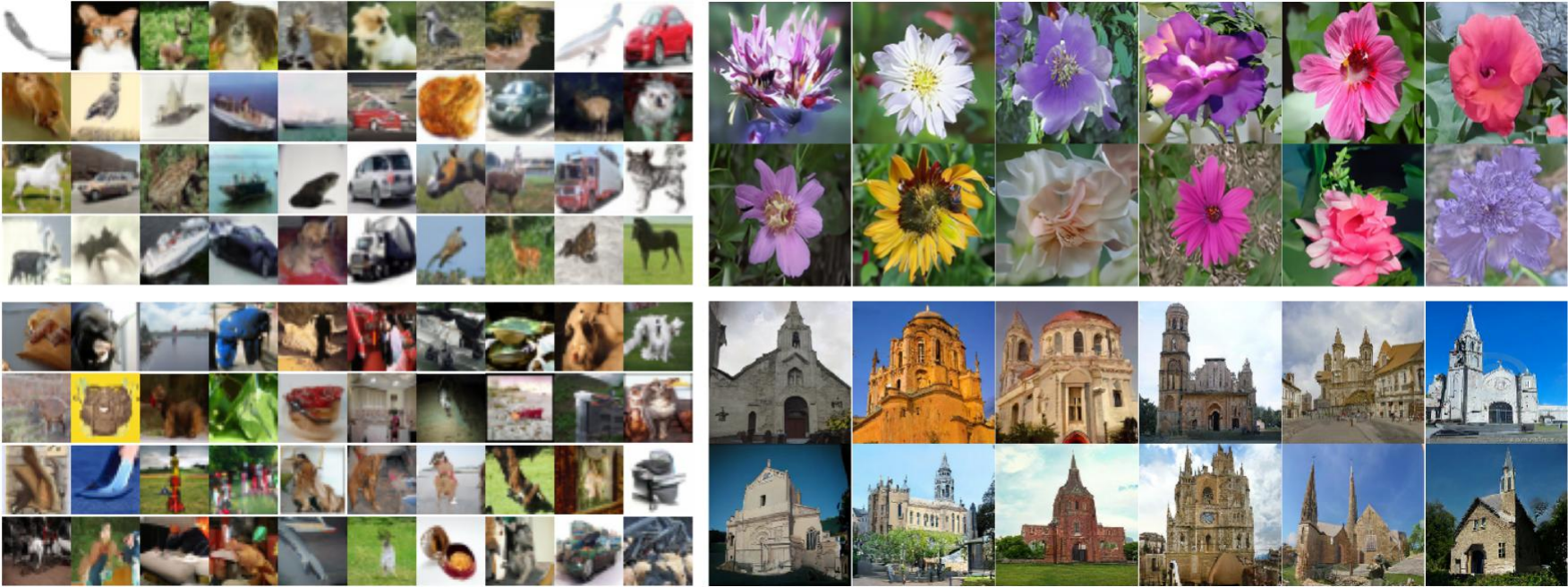}
    \caption{Unconditional image generation by \lfm{} on $32\times 32$ images:
    (upper left) CIFAR-10,
    (lower left) Imagenet-32;
    $128\times 128$ images:
    (upper right) Flowers,
    (lower right) LSUN Church.}
    \label{fig:image_gen}
\end{figure*}

\subsection{Two-dimensional data toy examples}\label{sec:2d_toy}

We consider a toy example where the task is to generate two-dimensional distributions with no analytic form: the ``tree'' and ``rose'' distributions (Figure \ref{fig:2d}). To ensure a fair comparison, we use an identical training scheme for all methods, including the choice of optimizers, batch size (we use a 1/10 batch size for JKO-iFlow to avoid excessively long wall-clock training time), and number of training batches. Further details about the experimental setup are provided in Appendix \ref{app:exp}. The accuracy of the trained models is evaluated using NLL. 
Figure \ref{fig:2d} shows that \lfm{} generates the distributions effectively and achieves slightly better NLL compared to all methods except JKO-iFlow, which, unlike \lfm{}, is not scalable to high-dimensional tasks.

\subsection{Tabular data generation}\label{sec:tabular}

We apply \lfm{} to a set of tabular datasets  \cite{papamakarios2017masked}, 
which are collected and processed from the University of California Irvine Machine Learning Repository. The datasets vary in dimensionality and application domains, such as household power consumption and carbon monoxide gas mixtures.
The generation performance is evaluated by test NLL 
as commonly used in the CNF literature  \cite{nflow_review,onken2021ot,xu2023normalizing}.
For each dataset, we parameterize the local sub-flows using fully connected networks with varying widths and depths, while keeping the total number of model parameters equal to that of the global flow model (i.e., InterFlow) for a fair comparison.
Additional experimental details can be found in Appendix \ref{app:exp}.
The results, shown in Table \ref{tab:tabular_results}, indicate that the proposed \lfm{} is among the top two best-performing methods across all datasets. We also include ablation study on the choice of $N$ and $\gamma_n$ in Appendix \ref{app:schedule}.

\begin{table*}
\centering
\caption{FID (lower is better) comparison of \lfm{} against InterFlow and FM under same model sizes. FIDs marked with $^*$ are quoted from the original publications FM \cite{lipman2023flow} and InterFlow \cite{albergo2023building}. Note that InterFlow uses a Trig interpolant, and FM uses an OT interpolant to train the global flow.}
\label{tab:fid}
\setstretch{1.15}
\resizebox{.95\textwidth}{!}{
\fontsize{12pt}{12pt}\selectfont
\begin{tabular}{lccc|ccc|ccc}
\hline
 & \multicolumn{3}{c|}{CIFAR-10} & \multicolumn{3}{c|}{Imagenet-32} & \multicolumn{3}{c}{Flowers-128} \\
 & FID & Batch size & \# of batches & FID & Batch size & \# of batches & FID & Batch size & \# of batches \\
\hline
{FM} 		& 6.35$^*$ 	& 256 &  391k 		& 5.02$^*$ 	&  1024 &  250k	 & 70.8 & {40} 	& 40k \\
InterFlow 		& 10.27$^*$ 	 & 400 	& 500k 			& 8.49$^*$ 	& 512 	& 600k 		 &  65.9 & {40} 	& 40k \\
{\lfm{} (OT)} 	&   6.80  	&  256 & 50k 	& 7.20 		& 256 	& 200k 		 & 55.7 & {40} 	& 40k  \\
\lfm{} {(Trig)} 	& 6.66   	& 256 & 50k 	& 7.00 		& 256 	& 200k		 & 59.7 & {40} 	& 40k \\
\hline
\end{tabular}}
\end{table*}

\subsection{Image generation}\label{sec:img_gen}

We apply \lfm{} to unconditional image generation of $32\times 32$ and $128\times 128$ images (meaning that we do not use class labels). We compare \lfm{} with InterFlow {and FM} in terms of FID before and after distillation.
To ensure fair comparison, in \lfm{} with multiple blocks, each block uses a Unet of reduced capacity, so the overall model parameter counts are comparable between \lfm{} and InterFlow/FM.
Additional details of the experimental setup are provided in Appendix \ref{app:exp}.

\begin{table}[!t]
\centering
\caption{FID of \lfm{} and InterFlow before and after distillation on Flowers $128 \times 128$. }\label{tab:distillation}
\vspace{-3pt}
\setstretch{1.25}
\resizebox{.5\linewidth}{!}{
\fontsize{20pt}{20pt}\selectfont
\begin{tabular}{lccc}
\hline & Pre-distillation & Distilled @ 4 NFEs &  Distilled @ 2 NFEs \\
\hline
\lfm{} & 59.7 & \textbf{71.0} & \textbf{75.2} \\
InterFlow & 59.7 & 80.0 & 82.4 \\
\hline
\end{tabular}}
\vspace{-12pt}
\end{table}

\textbf{$32\times 32$ images.} We use the CIFAR-10  \cite{krizhevsky2009learning} and Imagenet-32  \cite{deng2009imagenet} datasets. 
As shown in Table \ref{tab:fid}, 
\lfm{} requires significantly less computation during training compared to InterFlow and achieves better or comparable FID values.
\lfm{} also demonstrates improved training efficiency against FM:
the original FM paper \cite{lipman2023flow}
reported a better FID on Imagenet-32 using a larger batch size and longer training than our   \lfm{}  model.
Overall, \lfm{} gets comparable generation quality with the FM methods \cite{albergo2023building,lipman2023flow}
with less computation and memory usage, see also Table \ref{tab:cifar-more}(b).
Generated images by \lfm{} are presented in Figure \ref{fig:image_gen}.
We also implemented distillation of the \lfm{} to NFE = 5 (NFE stands for the Number of Function Evaluations).
The generated images by the distilled \lfm{} are shown in Figure \ref{fig:image_distill}, exhibiting an almost negligible reduction in visual quality.

\begin{figure*}
    \centering
    \includegraphics[width=\linewidth]{ 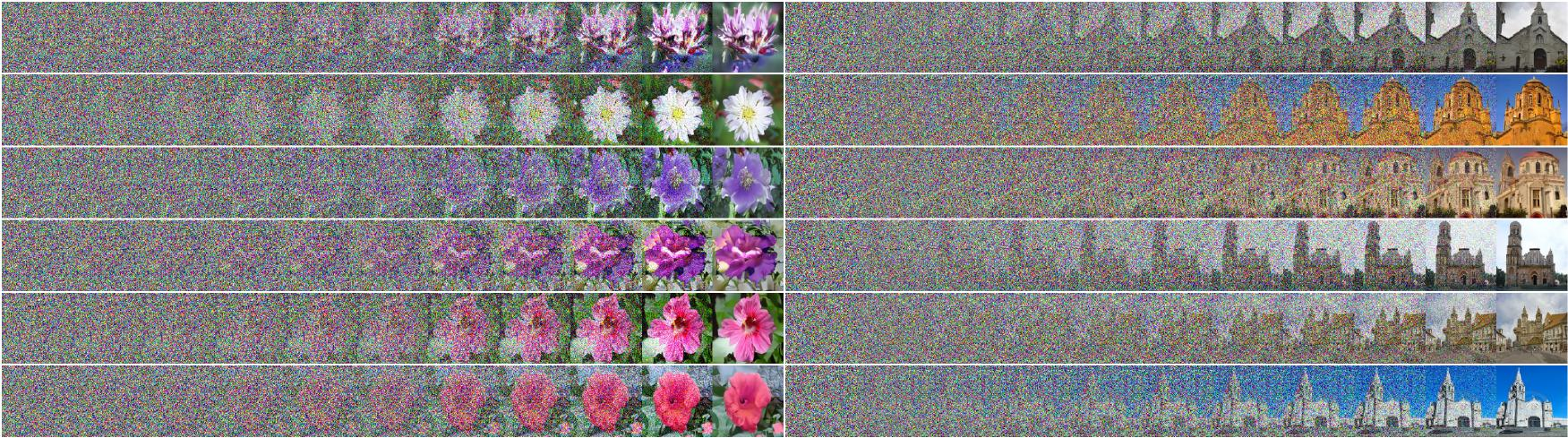}
    \caption{Noise-to-image trajectories by \lfm{}: Flowers (left) and LSUN Church (right).}
    \label{fig:image_gen_traj}
\end{figure*}

\textbf{$128\times 128$ images.} 
We apply \lfm{} model to the Oxford Flowers  \cite{Nilsback2008flower} 
and LSUN Church  \cite{yu2015lsun} datasets. 
The generated images are shown in Figure \ref{fig:image_gen},
and noise-to-image trajectories are shown in Figure \ref{fig:image_gen_traj}.
We compare with InterFlow/FM on the Flowers data,
where we first train \lfm{} and InterFlow{/FM} to achieve the same FIDs on the test set and the global flows require 1.25-1.5 times more training batches to reach the same performance.
Table \ref{tab:fid} shows lower FID  by \lfm{} using the same number of training batches.
To compare model performance after distillation, 
we distill \lfm{} using Algorithm \ref{algo:lfm_distill}, 
and the distill of InterFlow/FM follows the method in $K$-rectified flow  \cite{liu2023flow}. 
Table \ref{tab:distillation} shows the comparison with InterFlow (when pre-distilled  \lfm{} uses  Trig interpolant),
and Table \ref{tab:reflow} compares with $K$-rectified flow (when pre-distilled \lfm{} uses OT interpolant),
and in both tables \lfm{} achieves lower FID. 
Figure \ref{fig:flower_distill} highlights high-fidelity images generated by \lfm{} after distillation.

\begin{figure*}
    \centering
    \begin{minipage}{\textwidth}
        \includegraphics[width=\linewidth]{ 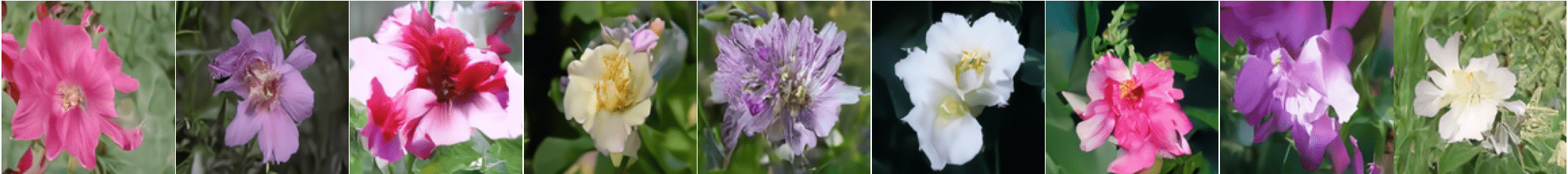}
        \subcaption{\lfm{} distillation @ 4 NFEs. FID = 71.0.} 
    \end{minipage}
    \hspace{5pt}
    \begin{minipage}{\textwidth}
        \includegraphics[width=\linewidth]{ 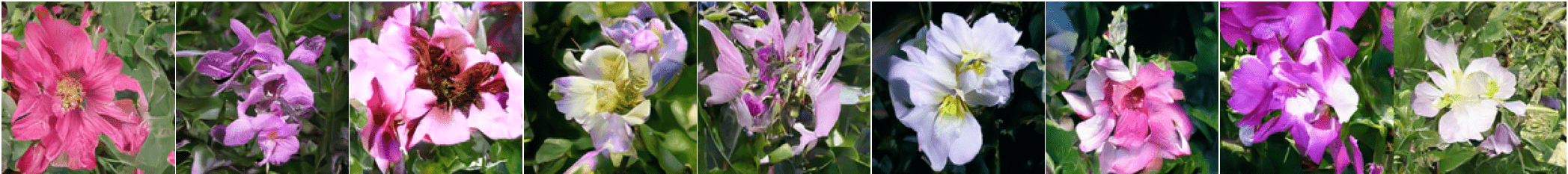}
        \subcaption{InterFlow distillation @ 4 NFEs. FID = 80.0.}
    \end{minipage}
    \caption{
    Qualitative comparison of \lfm{} and InterFlow after distillation.}
    \label{fig:flower_distill}
\end{figure*}

\subsection{Robotic manipulation policy learning}\label{sec:robotics}

We consider robotic manipulation tasks from the Robomimic benchmark  \cite{robomimic2021}, which includes 5 tasks involving the control of robot arms to perform various pick-and-place operations (Figure \ref{fig:robomimic}). For example, the robot may need to pick up a square object (Figure \ref{fig:robomimic_lift}) or move a soda can from one bin to another (Figure \ref{fig:robomimic_can}).
Recent generative models that output robotic actions conditioning on the state observations (e.g., robot positions or camera image embedding) have achieved the state-of-the-art performance on completing these tasks  \cite{chi2023diffusionpolicy}.
However, transferring pre-trained diffusion or flow models (on natural images) to the conditional generation task is challenging because the state observations are task-specific and contain nuanced details about the robots, objects, and environment. The details, absent in natural images, are critical for models to understand to determine the appropriate actions. 
As a result, it is often necessary to train a generative model from scratch to directly learn the relationship between task-specific state observations and actions.

We train a (global) FM model and the proposed \lfm{} from scratch, ensuring that the total number of parameters is identical for both methods. 
Additional experimental details can be found in Appendix \ref{sec:robomimic}.
As shown in Table \ref{tab:robomimic}, \lfm{} demonstrates competitive performance against FM in terms of success rate, with faster convergence in some cases (evidenced by higher success rates at early epochs). Consistent with findings in  \cite{chi2023diffusionpolicy}, the performance on the ``Toolhang'' task does not improve with extended training.

\begin{table}
\caption{
Success rate 
(defined in \eqref{eq:succ_rate}, higher is better) 
of FM and \lfm{} for robotic manipulation on Robomimic  \cite{robomimic2021}. Both methods are evaluated over 100 rollouts, and success rates are reported at different epochs in the format of (Success rate, epochs).
}\label{tab:robomimic}
\centering
\setstretch{1.2}
\resizebox{.65\linewidth}{!}{
\begin{tabular}{lccccc}
\toprule
& Lift & Can & Square & Transport & Toolhang \\
\hline 
FM & (1.00, 200)  & \makecell[l]{(0.94, 200) \\ (0.98, 500) } & \makecell[l]{(\textbf{0.88}, 200) \\ (\textbf{0.94}, 750) } & \makecell[l]{(0.60, 200) \\ (0.81, 1500) } & (0.52, 200)\\
\hline
\lfm{} & (1.00, 200)  & \makecell[l]{(\textbf{0.97}, 200) \\ (\textbf{0.99}, 500) } & \makecell[l]{(0.87, 200) \\ (0.93, 750) } & \makecell[l]{(\textbf{0.75}, 200) \\ (\textbf{0.88}, 1500)} & (\textbf{0.53}, 200)\\
\bottomrule
\end{tabular}
}\vspace{-10pt}
\end{table}

\section{Discussion}\label{sec:discuss}

There are several directions for future work. 
First, we currently assume that $T_n^{-1}$ is exact in the reverse process (generation). The analysis can be extended to incorporate the inversion error due to numerical computation in practice, possibly by following the strategy in   \cite{cheng2024convergence}.
It would also be of interest to relax Assumption \ref{assump:A1-A2-A3}, which imposes regularity conditions on the intermediate densities, including Gaussian-type tail bounds, regularity of the score, and integrability conditions on density ratios. To the best of our knowledge, assumptions of this kind are standard in the analysis of diffusion processes, Wasserstein gradient flows, and stability results for continuity equations, particularly when establishing convergence in KL or $\chi^2$ divergence rather than weaker metrics such as Wasserstein distance. 
In this work, we introduce Assumption \ref{assump:A1-A2-A3} to enable $\chi^2$ convergence (and hence in KL and TV),  but one may obtain sharper bounds by analyzing KL or TV directly and under weaker regularity assumptions.

Meanwhile, the proposed methodology can be further improved.
While the training of each block is based on flow matching, the main additional cost of \lfm{} compared to standard flow matching arises from the pushforward operation, which involves numerically integrating a trained flow model once per block.
In practice, this cost is modest for moderate-sized datasets and is incurred only once per block; it does not scale with the number of training iterations.
For example, on CIFAR-10, pushing the dataset forward requires only a few minutes per block, whereas on larger datasets such as ImageNet-32, we found that the pushforward operations account for approximately \(15\%\) of the total training time in our experiments.
We therefore view the sequential nature of the algorithm and the associated pushforward cost as practical limitations of the current approach, and believe they are worthwhile directions for future improvement.
A related question concerns model parameterization. In the current formulation, we learn the flow-matching blocks independently as \(\hat v(x,t;\theta_n)\). An interesting extension would be to introduce parameter sharing across blocks, which could impose additional temporal continuity in the flow model and potentially alleviate the need for sequential training.
Finally, it would be of interest to apply the proposed method to a broader range of applications, including more complex high-dimensional generative modeling tasks and learning-based control problems.

 \appendix

\setcounter{table}{0}
\setcounter{figure}{0}
\setcounter{algorithm}{0}
\renewcommand{\thetable}{A.\arabic{table}}
\renewcommand{\thefigure}{A.\arabic{figure}}
\renewcommand{\thealgorithm}{A.\arabic{algorithm}}
\renewcommand{\thelemma}{A.\arabic{lemma}}

\setcounter{assumption}{0} \renewcommand{\theassumption}{A.\arabic{assumption}}

\section{Proof of valid flow with dependent sampling}\label{app:proof-valid-v}

We introduce the following assumptions, essentially following \cite{albergo2023building}:

\begin{assumption}\label{assump:technical-rho01-It}
The two endpoints $(x_l, x_r) \sim \rho_{0,1} $ the joint density, where
\begin{itemize}
\item[(i)]  $\rho_{0,1}(x_0, x_1)$ is continuously differentiable,
with two marginals $p$ and $q$, i.e., 
$\int \rho_{0,1}(x_0, x_1) dx_1 = p(x_0)$
and 
$\int \rho_{0,1}(x_0, x_1) dx_0 = q(x_1)$.

\item[(ii)]
$I_t( x_l, x_r)$ is continuously differentiable in $(t, x_l, x_r)$, satisfies \eqref{eq:It-endpoints-condition} and 
$
\int_0^1 \E_{x_l, x_r} \| \partial_t I_t( x_l, x_r) \|^2 dt  < \infty.
$

\item[(iii)] For all $t \in [0,1]$,
\begin{align*}
& \int_{\R^d} \big|   \int_{\R^d \times \R^d} e^{i \xi \cdot I_t(x_0, x_1)} \rho_{0,1}(x_0, x_1) dx_0 dx_1  \big| d\xi < \infty, \\
& \int_{\R^d} \big|   \int_{\R^d \times \R^d}   \partial_t I_t(x_0, x_1) e^{i \xi \cdot I_t(x_0, x_1)} \rho_{0,1}(x_0, x_1)  
dx_0 dx_1  \big| d\xi < \infty.
\end{align*}

\end{itemize}

\end{assumption}

 \begin{proof}[Proof of Lemma \ref{lemma:FM_loss}]

     We write $v(\cdot, t)$ as $v_t(\cdot)$ and $\rho(\cdot, t)$ as $\rho_t(\cdot)$.
 We will explicitly construct $\rho_t$ and $v_t$, 
 and then show that
 (i) $\rho_t$ indeed solves the CE induced by $v_t$, 
 and (ii) $v_t$ is valid.

 Let  $\rho_t(x)$ be the concentration of the interpolant points $I_t(x_0, x_1)$ over all possible realizations of the two endpoints. 
 That is, using a formal expression with the Dirac measure $\delta$, we define
 \[
 \rho_t(x) : = \int_{\R^d} \int_{\R^d}  \delta( x - I_t(x_0, x_1)) \rho_{0,1}(x_0, x_1) dx_0 dx_1.
 \]
 Below, we will use the Dirac measure in more formal derivations, e.g., to express the probability current $j_t(x)$.
By the argument in \cite[Lemma B.1]{albergo2023building},
the integrability conditions in Assumption \ref{assump:technical-rho01-It}(iii) ensure that these formal derivations are mathematically meaningful as a result of the Fourier inversion theorem. 
This construction also means that $\rho_t$ is the marginal density of $x_t$ as stated in the lemma.

By \eqref{eq:It-endpoints-condition}, $I_0(x_0,x_1) = x_0$ and $I_1(x_0,x_1) = x_1$, 
and then the definition of $\rho_t$ gives that
\[
 \rho_0(x) = \int \rho_{0,1}(x, x_1) dx_1, 
 \quad 
 \rho_1(x) = \int \rho_{0,1}(x_0, x) dx_0.
 \]
Under Assumption \ref{assump:technical-rho01-It}(i), we know that 
  \[
 \rho_0 =  p, 
 \quad 
 \rho_1  =  q.
 \]
Meanwhile, taking $\partial_t$ of $\rho_t$ gives that 
 \begin{align}
 \partial_t \rho_t (x) 
& = 
 - \int_{\R^d} \int_{\R^d} \partial_t I_t(x_0, x_1) \cdot \nabla \delta( x - I_t(x_0, x_1)) 
 \rho_{0,1}(x_0, x_1) dx_0 dx_1  \nonumber \\
&  = -\nabla \cdot j_t(x), \label{eq:CE-jt-proof1}
 \end{align}
 where
 \[
 \begin{split}
 j_t(x)
 & : =  \int_{\R^d} \int_{\R^d} \partial_t I_t(x_0, x_1) \delta( x - I_t(x_0, x_1)) 
  \rho_{0,1}(x_0, x_1) dx_0 dx_1.
  \end{split}
 \]
 We now define $v_t$ to be such that 
 \[
 v_t(x) \rho_t(x) = j_t(x),
 \]
 this can be done by setting $v_t(x) = j_t(x)/\rho_t(x)$ if $\rho_t(x) > 0$ and zero otherwise. Then, \eqref{eq:CE-jt-proof1} directly gives that $\partial_t \rho_t = - \nabla \cdot ( \rho_t v_t  )$ which is the CE. This means that $v_t$ is a valid velocity field.

 To prove the lemma, it remains to show that the loss \eqref{eq:fm_loss} can be equivalently written as \eqref{eq:loss-FM-2}. 
 First,  the condition 
 $$
\int_0^1 \E_{x_l, x_r} \| \partial_t I_t( x_l, x_r) \|^2 dt  < \infty
$$
in Assumption \ref{assump:technical-rho01-It}(ii) implies that 
\begin{equation}\label{eq:int-condition-proof}
\begin{split}
& \int_0^1 \int \| v(x,t)\|^2 \rho_t(x) dx dt 
\le \int_0^1 \E_{x_l, x_r} \| \partial_t I_t( x_l, x_r) \|^2 dt  < \infty,
\end{split}
\end{equation}
following the same argument as in  \cite[Lemma B.2]{albergo2023building}.
\eqref{eq:int-condition-proof} ensures that the r.h.s. of \eqref{eq:loss-FM-2}  is well-defined.

 To see the equivalence between  \eqref{eq:fm_loss} and \eqref{eq:loss-FM-2}, 
 note that \eqref{eq:fm_loss}  can be written as 
 \begin{equation}\label{eq:loss-FM-1}
 \begin{split}
& L(\hat v) =  \int_0^1 l( \hat v , t ) dt, \\
& l( \hat v , t ) := \E_{x_0, x_1} \| \hat v_t ( I_t (x_0, x_1)) - \partial_t I_t( x_0, x_1) \|^2.
\end{split}
\end{equation}
For a fixed $t$, 
 \begin{align*}
 & l( \hat v , t ) 
  =  \int_{\R^d} \int_{\R^d} \| \hat v_t ( I_t (x_0, x_1)) - \partial_t I_t( x_0, x_1) \|^2 
\rho_{0,1}(x_0, x_1) dx_0 dx_1  \\ 
 & = \int_{\R^d} \int_{\R^d} \int_{\R^d} \| \hat v_t (x) - \partial_t I_t( x_0, x_1) \|^2 
 	\delta( x - I_t(x_0, x_1))  
\rho_{0,1}(x_0, x_1) dx_0 dx_1 dx \\
 & = 
 c_1(t) + \int_{\R^d} \int_{\R^d} \int_{\R^d} 
 ( \| \hat v_t (x)\|^2 - 2 \hat v_t (x) \cdot \partial_t I_t( x_0, x_1)  ) \\
&~~~~~~~~~~~~~~~~~~	 
\delta( x - I_t(x_0, x_1))  \rho_{0,1}(x_0, x_1) dx_0 dx_1 dx,
 \end{align*}
 where
 \[
 c_1(t):= \int_{\R^d} \int_{\R^d} 
  \|  \partial_t I_t( x_0, x_1)  \|^2  
  \rho_{0,1}(x_0, x_1) dx_0 dx_1,
 \]
 and $c_1(t)$ is independent from $\hat v$.
 We continue the derivation as
 \begin{align*}
 & l( \hat v , t )  - c_1(t) \\
 & = 
 \int_{\R^d}
 \| \hat v_t (x)\|^2
  \int_{\R^d} \int_{\R^d} 
 \delta( x - I_t(x_0, x_1))  
\rho_{0,1}(x_0, x_1) dx_0 dx_1 dx  \\
&~~~~
 - 2 \int_{\R^d} 
  \hat v_t (x) \cdot  
 \int_{\R^d} \int_{\R^d}  \partial_t I_t( x_0, x_1)  
\delta( x - I_t(x_0, x_1))  \rho_{0,1}(x_0, x_1) dx_0 dx_1 dx \\
 & = \int_{\R^d}
         \| \hat v_t (x)\|^2 \rho_t(x) dx
         - 2 \int_{\R^d} 
         \hat v_t (x) \cdot j_t(x) dx \\
 & =  \int_{\R^d}
         ( \| \hat v_t (x)\|^2 
         - 2 \hat v_t (x) \cdot v_t(x) ) \rho_t(x) dx \\
& =   \int_{\R^d}
         \| \hat v_t (x) -  v_t (x)\|^2 \rho_t(x) dx 
    - \int_{\R^d} \| v_t (x)\|^2   \rho_t(x) dx,        
  \end{align*}
and then, by defining \[
c_2(t): = \int_{\R^d} \| v_t (x)\|^2   \rho_t(x) dx,
\] 
which is again  independent from $\hat v$, we have
  \[
    l( \hat v , t )  = \int_{\R^d}  \| \hat v_t (x) -  v_t (x)\|^2 \rho_t(x) dx  +  c_1(t)  - c_2(t).
  \]
  Putting back to \eqref{eq:loss-FM-1} we have proved \eqref{eq:loss-FM-2} with the constant 
  $
  c =  \int_0^1 c_1(t) dt - \int_0^1  c_2(t) dt$,
  where the finiteness of the integrals is guaranteed by \eqref{eq:int-condition-proof}.
 \end{proof}

\begin{algorithm}
\caption{Distillation of \lfm{}}
\label{algo:lfm_distill}
\begin{algorithmic}[1]
\REQUIRE 
Samples $\sim q_n$, $n=0,\cdots, N$, generated by a pre-trained $N$-block \lfm{} 
\ENSURE{$N' = N/k$ distilled sub-models $\{T^D_n(\cdot)\}_{n=1}^{N'}$.}
\FOR{$n=1,\ldots,N'$}
\STATE Train $f(x; \theta^D_n)$ via $\min_{\theta_n^D} \mathbb{E}_{(x_n, x_{n-1})\sim (q_{N-kn},q_{N-k(n-1)})}$ 
 	$\|(x_n-x_{n-1})-f(x_{n-1}; \theta^D_n)\|^2$.
\STATE Output $T_n^D(x)=x+f(x; \theta^D_n)$.
\ENDFOR
\end{algorithmic}
\end{algorithm}

\section{Experimental details and additional results}\label{app:exp}

To train \lfm{} sub-flows, we use the Trig interpolant on 2D, tabular, and image generation experiments (Sections \ref{sec:2d_toy}--\ref{sec:img_gen}) and use the OT interpolant on robotic manipulation (Section \ref{sec:robotics}).
{During inference per block, we employ the Dormand-Prince-Shampine ODE sampler with tolerances of 1e-5 for 2d and tabular data, 1e-4 for 32x32 images, and 1e-3 for 128x128 images. We use 1 Euler step on the robotic manipulation experiments, as prior works have shown the inference efficiency of FM on such tasks  \cite{hu2024adaflow}.}

\subsection{Baseline descriptions}\label{app:baselines}

We selected the following baselines for comparison with the proposed \lfm{}, categorized into two groups: flow-based methods and non-flow-based methods.

\paragraph{\it Flow-based baselines}
{\bf Flow Matching:} Since \lfm{} is closely related to flow matching, we compare it against global flow matching methods that train a single large model. 
Specifically, we include 
FM  \cite{lipman2023flow} which uses the OT interpolant; 
InterFlow  \cite{albergo2023building} which uses the Trig interpolant; 
and ($K$-)Rectified Flow  \cite{liu2023flow} for distillation,
where 1-rectified Flow is the same as FM.
{\bf Normalizing Flow:} methods, such as flow matching, are also based on ODE formulations but train the velocity field via maximizing model likelihood rather than minimizing $\ell_2$ loss. We compare against a range of approaches in this category, including JKO-iFlow  \cite{xu2023normalizing}, OT-Flow  \cite{onken2021ot}, CPF  \cite{huang2021convex}, BNAF  \cite{de2020block}, and FFJORD  \cite{grathwohl2018ffjord}.

\paragraph{\it Non-Flow-based baselines}
   {\bf Diffusion Models:} Diffusion-based approaches have demonstrated strong generative performance across diverse tasks. Unlike flow models, which rely on ODE formulations and learn velocity fields, diffusion models are based on SDEs and learn score networks. We compare against ScoreSDE  \cite{song2021scorebased}, a widely adopted continuous-time framework, following the implementation in  \cite{huang2021variational}.
     To ensure a comprehensive comparison, we also include three deep learning-based density estimation methods:
    {\bf Roundtrip}
     \cite{liu2021density}, which leverages GANs to generate samples and estimate data density using importance sampling or Laplace approximation;
        {\bf AdaCat}  \cite{li2022adacat}, which employs adaptive discretization for autoregressive models to estimate complex continuous distributions;
        {\bf nMDMA}  \cite{gilboa2021marginalizable}, which combines learned scalar representations of individual variables through hierarchical tensor decomposition to estimate densities.

\subsection{2D, Tabular, and Image experiments}

In all experiments, we use the Adam optimizer
with the following parameters: $\beta_1=0.9, \beta_2=0.999, \epsilon=1e-8$; the learning rate is to be specified in each experiment.
Additionally, the number of training batches indicates how many batches pass through all $N$ sub-flows per Adam update.

On two-dimensional datasets and tabular datasets, we parameterize local sub-flows with fully connected networks; the dataset details and hyperparameters of \lfm{} are in Table \ref{tab:hyperparam_tabular}.

On image generation examples, we parameterize local sub-flows as UNets  \cite{nichol2021improved}, where the dataset details and training specifics are provided in Table \ref{tab:hyperparam_image}.
We adopt the Unet architecture in the OpenAI guided diffusion package \url{https://github.com/openai/guided-diffusion}, and the specific setup follows the InterFlow \cite{albergo2023building} and Flow Matching \cite{lipman2023flow} papers. 
For example, for Imagenet32, our Unet architecture follows that in Table 3 in \cite{lipman2023flow}, except that we halve the hidden dimension of each Unet to be 128 instead of 256. Because we use 4 blocks, our overall number of parameters is 196M, which is comparable to that of the (1-block) Unet therein.

Due to nondeterminism arising from random seeding, reduced-precision arithmetic (TF32), and cuDNN kernel selection, reported FID scores may vary by approximately $\pm 1$ across runs and hardware. In addition, variation can arise from differences in the FID computation pipeline. In this work, we use the \texttt{clean-fid} package \cite{parmar2021cleanfid}. For image experiments, we report a single-run FID obtained under our training setup.

\begin{figure}[!t]
    \centering
    \begin{minipage}{0.48\textwidth}
        \includegraphics[width=\linewidth]{ 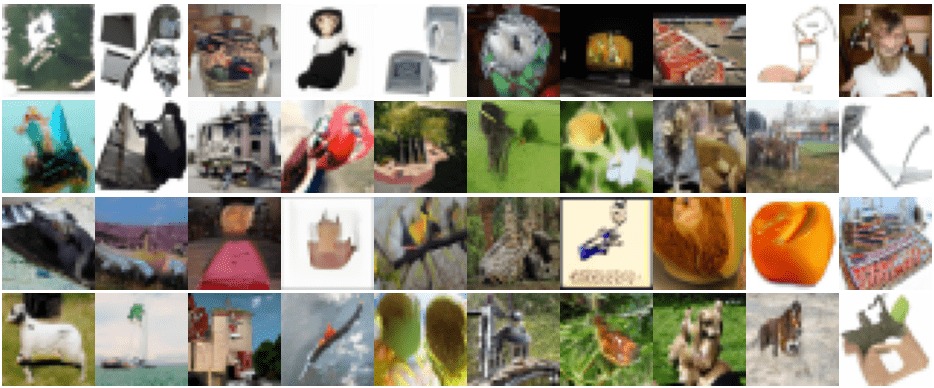}
        \subcaption{\lfm{} before distillation.} 
    \end{minipage}
    \hspace{5pt}
    \begin{minipage}{0.48\textwidth}
        \includegraphics[width=\linewidth]{ 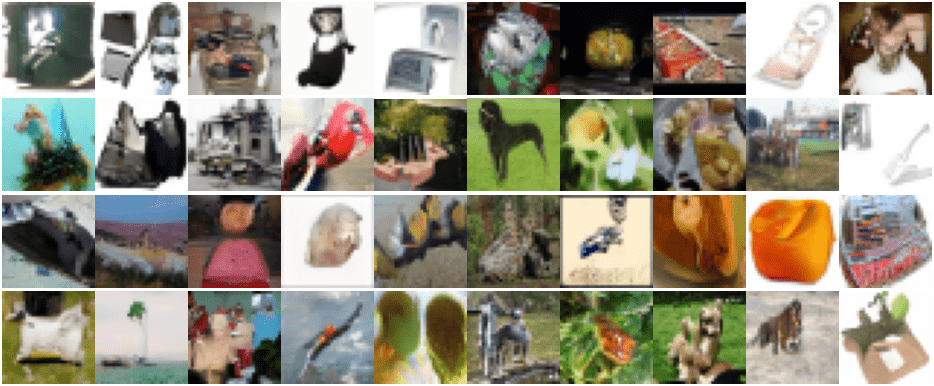}
        \subcaption{\lfm{} after distillation.}
    \end{minipage}
    \caption{Unconditional image generation on Imagenet-32 before and after distillation. We distill \lfm{} into a 5-NFE model.}

    \label{fig:image_distill}
\end{figure}
\begin{table}[!t]
\centering
\caption{
FID comparison of \lfm{} (1st column) 
and $K$-Rectified Flow  \cite{liu2023flow} (2nd-4th columns)
before and after distillation on Flowers $128 \times 128$. 
1-Rectified Flow is the same as FM \cite{lipman2023flow}.
See the first row for pre-distillation FIDs and the second row for FIDs after distillation.
The pre-distillation FID of $K$-Rectified Flow worsens with increasing $K$, as the training data is generated by $K-1$ many Rectified Flow.}\label{tab:reflow}
\setstretch{1.4}
\resizebox{.65\linewidth}{!}{
\begin{tabular}{lcccc}
\hline & \lfm{} & 1-Rectified Flow & 2-Rectified Flow & 3-Rectified Flow \\
\hline
Pre-distillation & 55.7 & 55.7 & 62.3 & 65.4 \\
Distilled@4NFEs & \textbf{76.9} & 84.6 & 82.8 & 82.7 \\
\hline
\end{tabular}}
\end{table}

\begin{figure*}
    \begin{minipage}{0.19\textwidth}
    \centering
        \includegraphics[width=0.75\linewidth]{ 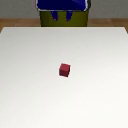}
        \includegraphics[width=0.75\linewidth]{ 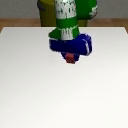}
        \subcaption{Lift}
        \label{fig:robomimic_lift}
    \end{minipage}
    \begin{minipage}{0.19\textwidth}
    \centering
        \includegraphics[width=0.75\linewidth]{ 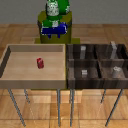}
        \includegraphics[width=0.75\linewidth]{ 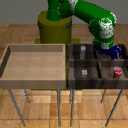}
        \subcaption{Can}
        \label{fig:robomimic_can}
    \end{minipage}
    \begin{minipage}{0.19\textwidth}
    \centering
        \includegraphics[width=0.75\linewidth]{ 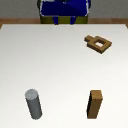}
        \includegraphics[width=0.75\linewidth]{ 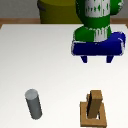}
        \subcaption{Square}
    \end{minipage}
    \begin{minipage}{0.19\textwidth}
    \centering
        \includegraphics[width=0.75\linewidth]{ 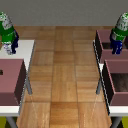}
        \includegraphics[width=0.75\linewidth]{ 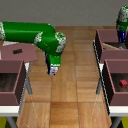}
        \subcaption{Transport}
    \end{minipage}
    \begin{minipage}{0.19\textwidth}
    \centering
        \includegraphics[width=0.75\linewidth]{ 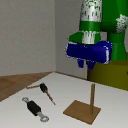}
        \includegraphics[width=0.75\linewidth]{ 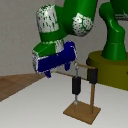}
        \subcaption{Toolhang}
    \end{minipage}

    \caption{Robotic manipulation on Robomimic  \cite{robomimic2021}. \textbf{Top row}: Initial conditions (IC). \textbf{Bottom row}: Successful completions. Each task starts from an IC and manipulates the robot arms sequentially to reach successful completion.}
    \label{fig:robomimic}

\end{figure*}

\subsection{Robotic manipulation policy learning}\label{sec:robomimic}

In the context of generative modeling, this task of robotic manipulation can be understood as performing sequential conditional generation. Specifically, at each time step $t\geq 1$, the goal is to model the conditional distribution $A_t|O_t$, where $O_t \in \R^O$ denotes the states of the robots at time $t$ and $A_t \in \R^A$ is the action that controls the robots. During inference, the robot is controlled as we iteratively sample from $A_t|O_t$ across time steps $t$.
Past works leveraging diffusion models have reached state-of-the-art performances on this task  \cite{chi2023diffusionpolicy}, where a neural network $v_{\theta}$ (e.g., CNN-based UNet  \cite{janner2022planning}) is trained to approximate the distribution $A_t|O_t$ via DDPM  \cite{ho2020denoising}. More recently, flow-based methods have also demonstrated competitive performance with faster inference  \cite{hu2024adaflow}.

We use the widely adopted \textit{success rate} to examine the performance of a robot manipulator:
\begin{equation}\label{eq:succ_rate}
    \text{Success rate}=\frac{\# \text{success rollouts}}{\# \text{rollouts}} \in [0,1].
\end{equation}
Specifically, starting from a given initial condition $O_1$ of the robot, each rollout denotes a trajectory $S=\{O_1, A_1, O_2, \ldots, A_T, O_T\}$ where $A_t|O_t$ is modeled by the generative. The rollout $S$ is a success if, at any $t \in {1,\ldots,T}$, the robotic state $O_t$ meets the success criterion (e.g., successfully pick up the square as in the task ``lift'' in Figure \ref{fig:robomimic_lift}).

\begin{table*}[!t]
\centering
\resizebox{\textwidth}{!}{
\begin{tabular}{lcccccc}
\hline & Rose & Fractal tree & POWER & GAS & MINIBOONE & BSDS300   \\
\hline Dimension & 2 & 2 & 6 & 8 & 43 & 63 \\
\# Training point & 2,000,000 & 2,000,000 &
{1,659,917} & 852,174 & 29,556 & 1,000,000 \\
\hline Batch Size & 10K & 10K & 30K & 50K & 1000 & {1000}  \\
Training Batches & 50K & 50K & 100K & 100K & 100K  & 30K \\
\makecell[l]{Hidden layer width \\ (per sub-flow)} & 256 & 256 & 256 & 362 & 362 & 512 \\
$\#$ Hidden layers & 3 & 3 & 4 & 5 & 4 & 4 \\
Activation & Softplus & Softplus & ReLU & ReLU & ReLU & ELU \\
$\#$ sub-flows $N$ & 9 & 9 & 4 & 2 & 2 & 4 \\
$(c,\rho)$
& (0.025, 1.25) & (0.025, 1.25) & (0.15, 1.3) & (0.05, 1) & (0.35, 1) & (0.25, 1) \\
\makecell[l]{Total $\#$ parameters in M\\(all sub-flows)} & 1.20 & 1.20 & 0.81 & 1.06 & 0.85 & 3.41 \\
Learning Rate (LR) & $0.0002$ & $0.0002$ & $0.005$ & $0.002$ & $0.005$ & $0.002$ \\
\makecell[l]{LR decay \\ (factor, frequency in batches)} & (0.99, 1000) & (0.99, 1000) & (0.99, 1000) & (0.99, 1000) & (0.9, 4000) & (0.8, 4000) \\
Beta $\alpha,\beta$, time samples & (1.0, 1.0)& (1.0, 1.0) & (1.0, 1.0) & (1.0, 0.5) & (1.0, 1.0) & (1.0, 1.0)  \\
\hline
\end{tabular}}
\caption{Hyperparameters and architecture for two-dimensional datasets and tabular datasets. The table is formatted similarly as  \cite[Table 3]{albergo2023building}.}
\label{tab:hyperparam_tabular}
\end{table*}

\begin{table*}
\centering
\resizebox{0.9\textwidth}{!}{
\begin{tabular}{lcccc}
\hline & CIFAR-10 & Imagenet-32 & Flowers & LSUN Churches   \\
\hline Dimension & 32$\times$32 & 32$\times$32 & 128$\times$128 & 128$\times$128 \\

\# Training point & 50,000 & 1,281,167 & 8,189 & 122,227 \\
\hline Batch Size 	& {256} & 256 & 40 & 40  \\
Training Batches 	& 50k 	& 200k & 40k  & 120k  \\
Hidden dim (per sub-flow) & 128 & 
 {128}  & {64} & 128 \\
$\#$ sub-flows $N$ & 4 &  {4}  & 4 & 4 \\\
$(c,\rho)$ 
&  {(0.1, 1.5)} & (0.3, 1.1) & (0.5, 1.5) & (0.4, 1.5) \\
Total $\#$ parameters in M (all sub-flows) & {143} &  {196} & {86} & 464  \\
Learning Rate (LR) & $0.0001$ & $0.0001$ & $0.0002$ & $0.0002$ \\
U-Net dim mult & [1,2,2,2,2] & [1,2,2,2] & [1,1,2,3,4] & [1,1,2,3,4]   \\
Beta $\alpha,\beta$, time samples & (1.0, 1.0) & (1.0, 1.0) & (1.0, 1.0) & (1.0, 1.0)   \\
Learned $t$ sinusoidal embedding & Yes & Yes & Yes & Yes\\
$\#$ GPUs & 1 & 1 & 1 & 1 \\
\hline
\end{tabular}}
\caption{Hyperparameters and architecture for image datasets. The table is formatted similarly as  \cite[Table 4]{albergo2023building}.}
\label{tab:hyperparam_image}
\end{table*}

\begin{table*}
\centering
\resizebox{0.9\textwidth}{!}{
\begin{tabular}{lccccc}
\hline & Lift & Can & Square & Transport & Toolhang   \\
\hline Batch Size & 256 & 256 & 256 & 256 & 256 \\
Training Epochs & 200 & 500 & 750 & 1500 & 200 \\
\makecell[l]{Hidden dims \\ (per sub-flow)} & [128,256,512] & [128,256,512] & [128,256,512] & [176,352,704] & [128,256,512] \\
$\#$ sub-flows $N$ & 4 & 4 & 4 & 2 & 4 \\\
$(c,\rho)$ 
& (0.15, 1.25) & (0.2, 1.25) & (0.2, 1) & (0.5, 1) & (0.25, 1.25) \\
\makecell[l]{Total $\#$ parameters in M \\ (all sub-flows)} & 66 & 66 & 66 & 67 & 67 \\
Learning Rate (LR) & $0.0001$ & $0.0001$ & $0.0001$ & $0.0001$ & $0.0001$ \\
$\#$ GPUs & 1 & 1 & 1 & 1 & 1\\
\hline
\end{tabular}}
\caption{Hyperparameters and architecture for robotic manipulation under a state-based environment on Robomimic  \cite{robomimic2021}.}
\label{tab:hyperparam_robot}
\end{table*}

We also describe details of each of the 5 Robomimic tasks below, including dimensions of observations $O_t$ and actions $A_t$, and the success criteria. Figure \ref{fig:robomimic} shows the initial condition and successful completion. Table \ref{tab:hyperparam_robot} contains the hyperparameter setting in each task, where we use the same network and training procedure as in  \cite{chi2023diffusionpolicy}.

\vspace{5pt}

\noindent
\textbf{Lift:} The goal is for the robot arm to lift a small cube in red. Each $O_t$ has dimension $16\times 19$, and each $A_t$ has dimension $16 \times 10$, representing state-action information for the next 16 time steps starting at $t$.

\vspace{5pt}

\noindent
\textbf{Can:} The robot aims to pick up a Coke can from a large bin and place it into a smaller target bin. Each $O_t$ has dimension $16\times 23$, and each $A_t$ has dimension $16 \times 10$, representing state-action information for the next 16 time steps starting at $t$.

\vspace{5pt}

\noindent
\textbf{Square:} The goal is for the robot to precisely pick up a square nut and place it onto a rod. Each $O_t$ has dimension $16\times 23$, and each $A_t$ has dimension $16 \times 10$, representing state-action information for the next 16 time steps starting at $t$.

\vspace{5pt}

\noindent
\textbf{Transport:} The goal is for the two robot arms to transfer a hammer from a closed container on one shelf to a target bin on another shelf. Before placing the hammer, one arm has to also clear the target bin by moving away a piece of trash to the nearby receptacle. The hammer must be picked up by one arm, and then handed over to the other. Each $O_t$ has dimension $16\times 59$, and each $A_t$ has dimension $16 \times 20$, representing state-action information for the next 16 time steps starting at $t$.

\vspace{5pt}

\noindent
\textbf{Toolhang:} The robot aims to assemble a frame that includes a base piece and a hook piece by inserting the hook into the base. The robot must then hang a wrench on the hook. Each $O_t$ has dimension $16\times 53$, and each $A_t$ has dimension $16 \times 10$, representing state-action information for the next 16 time steps starting at $t$.

\begin{table*}[t]
\centering
\begin{minipage}{0.45 \linewidth}
\centering
\begin{tabular}{   l r}
\hline
\lfm{} schedule &   \\
\hline
exponential (0.25,1) & -158.298 \\
exponential (0.5, 1) & -158.554 \\
exponential (0.15, 1.3) & -158.677 \\
diffusion linear & -157.080 \\
\hline
\end{tabular}
 \subcaption{(a) $N=4$}
\end{minipage}
\begin{minipage}{0.45 \linewidth}
\centering
\begin{tabular}{c c c}
\hline
$N$ & diffusion linear &  diffusion cosine \\
\hline
4 & -158.045 & -159.479\\ 
5 &  -157.3401 & -159.455\\
6 &  -157.4023 &  -159.191\\
8 &  -157.899 & -158.632\\
\hline
\end{tabular}
 \subcaption{(b) Varying $N$ using diffusion schedules}
\end{minipage}
\caption{
Test NLL (lower is better) on the BSDS300 dataset with varying schedules.}
\label{tab:ablation-bsds}
\end{table*}

\subsection{Ablation on discrete-time schedules}\label{app:schedule}

We study the effect of the discrete time discretization on model performance by varying both the number of blocks \(N\) and the block step sizes \(\{\gamma_n\}_n\).
Recall that in our formulation each block corresponds to an Ornstein-Uhlenbeck forward corruption with step size \(\gamma_n\), i.e., 
the target density $p_n^* = (\text{OU})_0^{\gamma_n} p_{n-1}$,
and the final block is followed by setting $p_N^* = \calN(0,I)$, thus our $\gamma_N = \infty$.
We consider three types of schedules:

\begin{itemize}
\item
Exponential schedule.
Our default choice is an exponential schedule
\[
\gamma_n = c\,\rho^{\,n-1}, \qquad n=1,\dots,N,
\]
with constants \(c>0\) and \(\rho \ge 1\).
This schedule allocates smaller steps at early blocks and increasingly larger steps at later blocks, resulting in a rapid decay of the signal-to-noise ratio (SNR) near the terminal blocks.
This includes the constant schedule by setting $\rho =1$.

\item
Cosine diffusion schedule.
We adapt the cosine noise schedule from \cite{nichol2021improved}, originally proposed for diffusion models, to our discrete-block setting.
Let
\[
\bar{\alpha}_n
= \frac{\cos^2\!\bigl(\tfrac{(n/N+s)\pi}{2(1+s)}\bigr)}
       {\cos^2\!\bigl(\tfrac{s\pi}{2(1+s)}\bigr)}, \qquad n=0,\dots,N,
\]
with offset \(s=0.008\), and define \(\alpha_n = \bar{\alpha}_n / \bar{\alpha}_{n-1}\).
We convert this to OU step sizes via \(\gamma_n = -\tfrac{1}{2}\log \alpha_n\).
This schedule concentrates most of the noise injection in later blocks while keeping early steps relatively mild.
For instance, when $N=4$,
\[
\{\gamma_n\}_{n=1}^{3} = \{0.0830,\,0.2697,\,0.6153\}, 
\]
and when $N=6$,
\[
\{\gamma_n\}_{n=1}^{5} = \{0.0374,\,0.1112,\,0.2042,\,0.3475,\,0.6590\}.
\]

\item Linear diffusion schedule.
The linear \(\beta\)-schedule commonly used in diffusion models  \cite{ho2020denoising,song2021scorebased} are primarily designed for large $N$, typically  $10^2 \sim 10^3$. 
While our $N$ is small (up to 10), we adapt the schedule here as another diffusion-inspired baseline.
Specifically, we linearly interpolate \(\beta_n \in [\beta_{\min}/N,\,\beta_{\max}/N]\), 
set \(\alpha_n = 1-\beta_n\), and again define \(\gamma_n = -\tfrac{1}{2}\log \alpha_n\).
This schedule increases noise approximately linearly in diffusion time.
We let $\beta_{max} = \min\{ 0.99N, 20\}$, and $\beta_{min} = 0.1$. When $N=4$, we have
\[
\{\gamma_n\}_{n=1}^{3} = \{0.0127,\,0.2128,\,0.5518\}, 
\]
and when $N=6$,
\[
\{\gamma_n\}_{n=1}^{5} = \{0.0084,\,0.1187,\,0.2604,\,0.4590,\,0.7932\}.
\]
\end{itemize}


To study the effect of the discrete-time schedule, we first conduct ablation experiments on the BSDS300 tabular dataset (tabular data), where all other training and model configurations are kept fixed.
We fix the number of blocks to \(N=4\) and compare different choices of schedules, and the results are reported in Table~\ref{tab:ablation-bsds}(a).
We further evaluate the two diffusion schedules (linear and cosine) with varying numbers of blocks \(N \in \{4,5,6,8\}\), and the results are shown in Table~\ref{tab:ablation-bsds}(b).
Across all settings, the resulting NLLs are close to those reported in Table~\ref{tab:tabular_results}, indicating that performance on this dataset is relatively insensitive to the choice of schedule.

\begin{table*}[t]
\begin{minipage}{0.375 \linewidth}
\centering
\begin{tabular}{l r}
\hline
\lfm{} schedule &  FID\\
\hline
exponential (0.1, 1.5) &            6.80 \\
exponential (0.2, 1.5) &            7.89\\
exponential (0.3, 1.1)          &  8.55  \\
diffusion cosine 	          & 6.07	\\
diffusion linear 	         &  	23.16 \\
\hline
\end{tabular}
\subcaption{(a) FID on CIFAR-10, $N=4$ }
\end{minipage}
\begin{minipage}{0.495 \linewidth}
\centering
\begin{tabular}{l c c }
\hline
Model &  Memory (allocate, peak) & Total \# parameters\\
\hline
LFM $N=4$  & 0.91, 14.45 & 142.99 \\
FM   ($N=1$) & 2.80, 28.03 & 142.86\\
\hline
\end{tabular}
\subcaption{(b) GPU memory print (GB) and model parameter count (M)}
\end{minipage}
\caption{\label{tab:cifar-more}
(a) FID (lower is better) on CIFAR10 with varying schedules.
(b) Memory print on a single GPU during training of a single block. 
The LFM model has 4 blocks, and each has a Unet hidden dimension of 128.
FM model has one block with Unet hidden dimension 256.}
\end{table*}

We additionally conduct experiments on the image dataset CIFAR-10 to examine the effect of the schedule. 
The results are summarized in Table~\ref{tab:cifar-more}(a).
We use $N=4$, the linear (OT) interpolation function, and the same setup as before.
Among all tested schedules,
the exponential schedule with parameters \((c,\rho)=(0.1,1.5)\) achieves the best FID of 6.8.
Other exponential variants and the cosine diffusion schedule yield slightly worse performance, with FID increases within 2.
In contrast, the linear diffusion schedule performs substantially worse at this small number of blocks.
We attribute this behavior to the interaction between coarse discretization and noise allocation.
When \(N\) is very small, the linear schedule concentrates a large amount of noise in later blocks, which appears to be unfavorable for image generation. The cosine schedule, by contrast, allocates noise more smoothly across blocks and exhibits better robustness under coarse discretization.
The GPU memory prints for \lfm{} (N=4) and FM (N=1, doubled Unet base channel) are given in Table~\ref{tab:cifar-more}(b).

\section*{Acknowledgment}

The authors would like to thank Edward Chen and Junghwan Lee for help with numerical experiments.
The authors also thank the anonymous referees and the Associate Editor for their constructive feedback. 
YX and CX acknowledge  NSF CAREER CCF-1650913, NSF DMS-2134037, CMMI-2015787, CMMI-2112533, DMS-1938106, DMS-1830210, ONR N000142412278, and the Coca-Cola Foundation.
XC acknowledges
NSF DMS-2237842
and 
Simons Foundation MPS-MODL-00814643.

\bibliographystyle{IEEEtran}
\bibliography{references}
\end{document}